%% file: main.tex
\documentclass[twoside]{article}

\pdfoutput=1

\usepackage[accepted]{aistats2019}
\usepackage{natbib}
\renewcommand{\bibname}{References}
\renewcommand{\bibsection}{\subsubsection*{\bibname}}

\usepackage[utf8]{inputenc} 
\usepackage[T1]{fontenc}    
\usepackage{hyperref}       
\usepackage{cleveref}
\usepackage{url}            
\usepackage{booktabs}       
\usepackage{amsfonts}       
\usepackage{nicefrac}       
\usepackage{microtype}      
\usepackage{enumitem}
\usepackage{wrapfig}
\usepackage{makecell}

\usepackage{siunitx}
\usepackage[version=4]{mhchem}
\usepackage{collcell}
\usepackage{amsmath}
\usepackage{amssymb}
\usepackage{amsthm}
\usepackage{amsfonts}

\newcommand{\cmark}{{\color{teal}\checkmark}}
\newcommand{\xmark}{{\color{red} \boldsymbol{\times}}}

\usepackage{multirow}

\usepackage{tikz}

\usepackage{thmtools}
\usepackage{thm-restate}

\declaretheorem[name=Theorem]{thm}

\newtheorem{example}[thm]{Example}
\newtheorem{corollary}[thm]{Corollary}

\newtheorem{remark}[thm]{Remark}
\newtheorem{lemma}[thm]{Lemma}

\usepackage{caption,subcaption}

\newcommand{\customlabel}[2]{%
\protected@write \@auxout {}{\string \newlabel {#1}{{#2}{}}}}
\makeatother

\usepackage[textsize=footnotesize,color=green!40]{todonotes}

\newcommand{\vct}{ }

\newcommand{\argmin}{\mathop{\mathrm{argmin}}}
\newcommand{\argmax}{\mathop{\mathrm{argmax}}}

\def\E{\mathbb{E}}
\def\P{\mathbb{P}}

\def\cA{\mathcal{A}}

\def\cD{\mathcal{D}}

\def\cX{\mathcal{X}}

\newenvironment{talign}
 {\let\displaystyle\textstyle\align}
 {\endalign}
\newenvironment{talign*}
 {\let\displaystyle\textstyle\csname align*\endcsname}
 {\endalign}

\begin{document}

%
\runningtitle{Imitation-Regularized Offline Learning}

%

\twocolumn[

\aistatstitle{Imitation-Regularized Offline Learning}

\aistatsauthor{ Yifei Ma \And Yu-Xiang Wang \And Balakrishnan (Murali) Narayanaswamy }

\aistatsaddress{ yifeim@amazon.com \And  yuxiangw@cs.ucsb.edu\footnotemark \And muralibn@amazon.com } ]

\begin{abstract}
We study the problem of offline learning in automated decision systems under the contextual bandits model. We are given logged historical data consisting of contexts, (randomized) actions, and (nonnegative) rewards.
A common goal is to evaluate what would happen if different actions were taken in the same contexts, so as to optimize the action policies accordingly. The typical approach to this problem, inverse probability weighted estimation (IPWE) \citep{bottou2013counterfactual}, requires logged action probabilities, which may be missing in practice due to engineering complications. Even when available, small action probabilities cause large uncertainty in IPWE, rendering the corresponding results insignificant. To solve both problems, we show how one can use policy improvement (PIL) objectives, regularized by policy imitation (IML). We motivate and analyze PIL as an extension to Clipped-IPWE, by showing that both are lower-bound surrogates to the vanilla IPWE.
We also formally connect IML to IPWE variance estimation \citep{swaminathan2015counterfactual} and natural policy gradients.
Without probability logging, our PIL-IML interpretations justify and improve, by reward-weighting, the state-of-art cross-entropy (CE) loss that predicts the action items among all action candidates available in the same contexts.
With probability logging, our main theoretical contribution connects IML-underfitting to the existence of either confounding variables or model misspecification.
We show the value and accuracy of our insights by simulations based on Simpson's paradox, standard UCI multiclass-to-bandit conversions and on the Criteo counterfactual analysis challenge dataset.
\end{abstract}

\input{ma2018sim.tex}

\subsubsection*{Acknowledgements}

We appreciate Haibin Lin for help with mxnet sparse matrix operators, Yuyang (Bernie) Wang, Tengyang Xie for detailed discussions, Adith Swaminathan for insights about the Criteo counterfactual challenge dataset, and the anonymous reviewers for their constructive comments.




\bibliographystyle{plain}
\bibliography{ma2018sim}

\onecolumn

\appendix

\input{ma2018appendix}

\end{document}

%% file: ma2018sim.tex
\section{Introduction}
\label{sec:introduction}

There are two types of offline learning approaches in automated decision systems (e.g. recommendation systems): Q-learning and policy learning. Q-learning uses reward-modeling (or supervised learning) to predict rewards from both the context features and the action features. Formally, we estimate $Q(a,x)$, the expected reward from taking an action $a$ in context $x$; decisions are then implied by the \emph{greedy policy} that selects actions with the highest expected reward in each decision context \citep{juan2016field,rendle2012factorization,hidasi2015session,munos2016safe}.
Reward modeling suffers from biases due to unobserved confounding variables or model 
mis-specification. For example, items that are temporarily popular because of sales events may not be popular in general, but these sales can confuse the learning system by reinforcing any mistakes in the previous policies when they are mistaken as the causes of successful sales.
Therefore, it is often desirable to build reward-model-free decision systems that directly estimate the \emph{causal effects} of the candidate actions, robust to hidden biases in previous logging policies.

%
%


    

\footnotetext{Most of the work done while at Amazon.}

As a result, many decision systems use policy learning \citep{bottou2013counterfactual,swaminathan2015counterfactual,lefortier2016large,joachims2018deep,austin2011introduction,dudik2011doubly,horvitz1952generalization}.
To directly optimize for the decision policy in the presence of confounders, one additional requirement is to have randomization in the logging process: every candidate action must have nonzero probability to be selected given any context.
By logging these action probabilities, unbiased causal effects can be estimated via inverse probability weighted estimator (IPWE), which up- or down-weights the rewards according to the odds of choosing the same action in the same context, across the two policies.

Unfortunately, accurate probability logging is a significant practical challenge.
More worryingly, even with probability logging, naive IPWE suffers from large variance in estimation of causal effects, due to the up-weighting of rare actions, some of which will, on balance, appear in the logged datasets at least once. For example, consider a logging policy with a $1\%$ chance of sampling a rare action. The rare action will be included in a dataset of a hundred samples at least once with probability $1-(1-1\%)^{100}=63\%$. When this happens, IPWE weighs the single item as much as $100$ examples-- equivalent to half of the dataset. This increases the variance of any estimates that depend on that example significantly
(See Example~\ref{eg:epsilon-greedy} in the appendix for more details and Figure~\ref{fig:IPWE-weight} for a real-world example).
Since this problem is caused by rare actions, one solution is to use biased estimators that conservatively estimate any potential lifts after up-weighting, \citep{bottou2013counterfactual,swaminathan2015counterfactual,schulman2017proximal,schulman2015trust}, which we show corresponds to estimating a lower bound on the eventual policy improvements.

\begin{table}[t]
    \centering
    \caption{Challenges tackled by different objectives}
    \label{tab:challenges}
    \begin{tabular}{lcccc}\toprule
    challenge & Q & IPWE & IML & PIL-IML \\\midrule
    confounders & $\xmark$ & $\cmark$ & $\cmark$ & $\cmark$ \\
    small/no probs & $\cmark$ & $\xmark$ & $\cmark$ & $\cmark$ \\
    improvement & $\cmark$ & $\cmark$ & $\xmark$ & $\cmark$ \\\bottomrule
    \end{tabular}
\end{table}

\begin{figure}[t]
    \centering
    \includegraphics[width=0.9\linewidth]{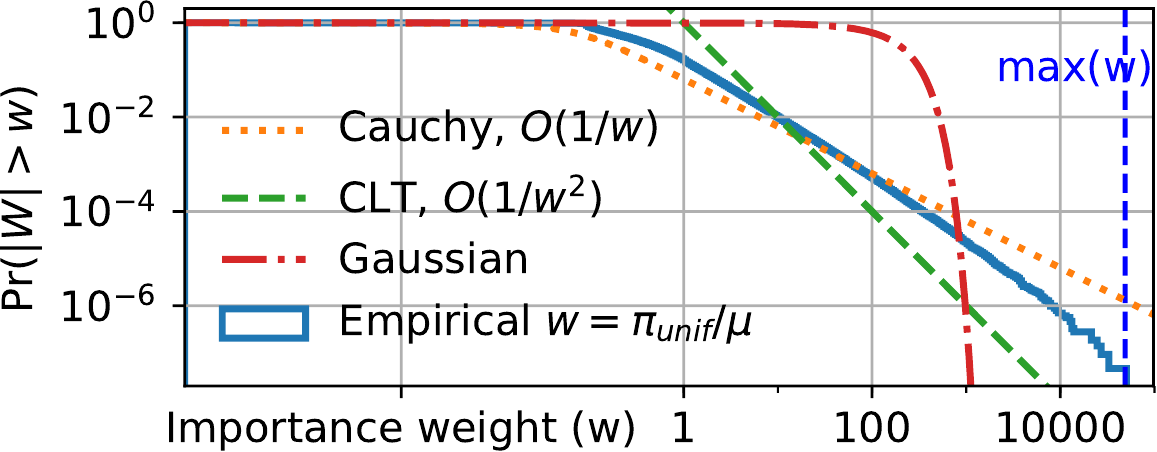}
    \caption{The importance weight distribution in Criteo counterfactual-analysis dataset \citep{lefortier2016large} has unbounded variance due to its slower-than-central-limit-theorem (CLT) tail; $\max(w)=\num{49000}$ in a total of 21MM examples.
    }
    \label{fig:IPWE-weight}
    \vspace{-1em}
\end{figure}

In this paper, we show a connection between policy improvement lower bounds (PIL) and Clipped-IPWE \citep{bottou2013counterfactual} (Theorem~\ref{thm:gap}).
This connection opens up a number of extensions to Clipped-IPWE, and we focus on one in particular - the log-transformed-IPWE. We analyze this estimator, and further establish connections to policy gradients (PG) \citep{sutton2000policy} using log-separability and Taylor approximations. In essence, we show that PG for contextual settings is equivalent to the cross-entropy (CE) objective in multi-class classification, where the label is whichever action that leads to the largest positive rewards in a particular context (Eq. \ref{eq:ce-pg}). Since PG/CE does not require logged action probabilities, this provides a justification for their success even when logging is biased, particularly when compared with other offline learning objectives, e.g. Bayesian personalized ranking, sigmoid or triplet losses.

Once we identify PG/CE as an approximation we propose and analyze policy imitation learning (IML) as a regularizer \eqref{eq:iml-kl}, and show that this improves the tightness of estimates (Theorem \ref{thm:iml-and-ipwe-variance}). We connect IML to IPWE variance estimation \citep{swaminathan2015counterfactual}. In our experiments we see that IML is superior because it does not rely on unstable IPWE mean estimates, which is required for direct variance estimation (Section \ref{sec:criteo}).
We also connect IML to natural policy gradients \citep{kakade2002natural,schulman2017proximal,schulman2015trust,munos2016safe} without requiring knowledge of the model families.
Similar to PG/CE, IML also works without logged action probabilities.
The combined PIL-IML objective predicts the best next action that is ever taken in the logged data, with a weight that is large for very positive rewards and small but still positive for less-good rewards.

Finally, we show that when we have logged action probabilities, we can still benefit from IML by using it to diagnose a common problem in offline learning - when the logging policy is not in the class of optimization policies under consideration when learning. We show that IML-underfitting implies that the learning policy class does not have enough complexity or sufficient decision variables to imitate the original policy, which may lead to model biases.
On the other hand, IML-underfitting can be used to our advantage by pointing out where we should collect additional data, through better action explorations (Theorem~\ref{thm:entropy}).

Notice, IML is different from propensity fitting, which is used as a plug-in replacement for the logging probabilities in the denominators of IPWE \citep{ dudik2011doubly,strehl2010learning}.
On the other hand, we extend our methods to doubly robust approaches \citep{robins1994estimation,robins1995semiparametric,bang2005doubly,jiang2015doubly,dudik2011doubly} and switching approaches \citep{wang2016optimal,kang2007demystifying,thomas2016data} for additional, free variance reduction.

\section{Offline Learning Objectives}

While many methods have been proposed for learning from logged data, it is often unclear what the objective being maximized or minimized by different approaches. We introduce some clarifying definitions here. Let $(\vct x_1,\vct a_1,\mu_1,r_1), ..., (\vct x_n,\vct a_n,\mu_n,r_n)\in \cX\times \cA\times \mathbb{R} \times \mathbb{R}$ be a dataset containing $n$ sample points of context, action, (possibly missing) probability of action given context, and reward, collected while running an automated decision system under a \emph{logging policy}. Both contexts and actions have  feature representations. For tabular actions, we use their indicator vectors, $\vct a=\vct e_a$.

Specifically, the data are generated i.i.d. as:
$
x_i,h_i\sim P(x,h),
a_i\sim \mu(a|x_i,h_i) \text{ supported on } \cA(x_i,h_i) \subset \cA,
\mu_i = \mu(a|x_i,h_i),
r_i\sim  P(r|x_i,h_i,a_i)
$
Here, $h_i$ is an unobserved confounding variable that affects both the action $a_i$ and the reward $r_i$. 
$\cA(x_i,h_i)$ is discrete set of candidate actions available in context $(x_{i},h_{i})$.
We can assume that $\cA(x_i,h_i)$ is logged for every $i = [n]$. 
Note that most existing work implicitly assume that $\mu(a|x_i) = \mu(a|x_i,h_i)$, but it is common that certain decision variables are not logged, especially in publicly available data sets, due to proprietary features or human operators overwriting decisions every once in a while. In general, we do not assume that we know the analytic form of $\mu$ besides having logged $\mu_i$ for the specific action taken. We will also consider the setting when even $\mu_i$ is not known, which makes it fundamentally impossible to do consistent off-policy evaluation (due to confounders), but we will show that often we can still do off-policy learning and adapt to the unknown propensities.
To the best of our knowledge, this is the first time that a result of such flavor is presented.

The task of \emph{offline learning} is to come up with a new policy, which is a distribution over candidate actions given context, $\pi(\vct a\mid\vct x),\forall \vct a\in\cA(\vct x)$, such that the expected reward 
under $\pi$:
\begin{equation}
    \mathbb{E}_{\pi}r
    = 
\mathbb{E}
\sum\nolimits_{\vct a\in\mathcal{A}(\vct x)} \pi(\vct a\mid\vct x) r(\vct x,\vct a).
\label{eq:objective}
\end{equation}
is as large as possible. This is difficult because it aims to estimate rewards for actions that may not have been logged in a particular context,
unless they happen to coincide with the randomized action choices.

\textbf{Inverse-probability weighted estimation} (IPWE) \citep{horvitz1952generalization,bottou2013counterfactual} is an unbiased offline evaluation method, which uses importance weights to estimate expectations under any new policy with samples generated from the original logging policy,
\begin{align}
	&\mathbb{E}_\pi r
    =\mathbb{E} \sum\nolimits_{\vct a\in\mathcal{A}(\vct x)}
    \pi(\vct a\mid\vct x)
    r(\vct x,\vct a)
    \label{eq:IPWE}
    \\
    &=\mathbb{E}
    \sum\nolimits_{\vct a\in\mathcal{A}(\vct x)}
    \mu(\vct a\mid\vct x)
    \biggl[
   \frac{\pi(\vct a\mid\vct x)}{\mu(\vct a\mid\vct x)}
    r(\vct x,\vct a)
    \biggr]
   = \mathbb{E}_{\mu}
   \biggl[
    \frac{\pi}{\mu}
  r
  \biggr],
  \nonumber
\end{align}
where the last expectation is over the logging policy and can be estimated (without bias) by its sample mean.
Define $w=\frac{\pi}{\mu}$ and $w_i=\frac{\pi(\vct a_i\mid\vct x_i)}{\mu_i}$ to be the importance weight in function form and instance form, respectively.
Empirically, unbiased policy improvement can be maximized by
\begin{equation}
	\max_{\pi} \Delta {\rm IPWE}(\pi) = 
    \frac{1}{n}\sum\nolimits_{i=1}^n (w_i-1) r_i.
    \label{eq:ipwe}
\end{equation}

\textbf{The variance of policy improvements} is:
\begin{equation}
    \mathbb{V}(\Delta {\rm IPWE}(\pi))
    = \frac{1}{n}\mathbb{V}_\mu((w-1)r),
\end{equation}
which can similarly be estimated from the logged data.

While IPWE does not model the reward function and thus avoids modeling biases,
it depends on randomized sampling of actions, which are usually the result of exploration/exploitation trade-offs.
Lack of sufficient exploration, very common in practice, may lead to a large variance in the estimate, because data points with small action probabilities have large weights.

Any \textbf{objective} then must consider the trade-off between bias and variance. Our objective, which is often reasonable from a practical perspective, is to reliably maximize the policy improvement by a significant margin over the logging policy.

\section{Proposed methods}

To solve the large variance in IPWE, we propose to maximize a policy improvement lower-bound (PIL) regularized by policy imitation learning (IML). Assuming that the rewards are nonnegative, the general form is:
\begin{equation}
    \max_{\pi} {\rm PIL}(r,\pi) + \epsilon {\rm IML}(\pi),
    \label{eq:pil-iml}
\end{equation}
where $\epsilon$ is a tuning parameter
to trade-off exploration/exploitation.
One notable special case of the objective resembles a reward-weighted cross-entropy (CE) objective, written as:
\begin{equation}
    \argmin_\pi \sum\nolimits_{i=1}^n
    \Bigl[
    r_i\log\frac{1}{\pi(a_i\mid x_i)} + \epsilon\log\frac{1}{\pi(a_i\mid x_i)}
    \Bigr].
    \label{eq:ce}
\end{equation}
We thus see the use of CE loss in offline learning as the result of a particular choice in the trade-offs of bias and variance (and ease of generalization and optimization).

\subsection{Policy improvement lower-bounds (PILs)}

\begin{wrapfigure}{r}{0.5\linewidth}
    \centering
    \vspace{-1em}
    \begin{tikzpicture}[domain=-0.0:3.5, scale=0.6]
    \draw[very thin,color=gray] (-0.1,-0.1) grid (2.8,2.8);
    \draw[->] (-0.2,0) -- (3.5,0) node[right] {$w=\frac{\pi}{\mu}$};
    \draw[->] (0,-0.2) -- (0,3.2) node[above] {$\bar w$};
    
    \draw[color=blue, thick, domain=0:1] plot(\x, \x);
    \draw[color=blue, very thick, dotted, domain=0.3:1] plot(\x,{ln(\x)+1});
    \draw[color=blue, thick, domain=1:3.5] plot(\x,{ln(\x)+1})
        node[right] {PIL};
    \draw[color=orange, dashed, very thick, domain=2.8:3.5] plot(\x,{2.8})
        node[right] {IPWE$_\tau$};
    \draw[color=red, dotted, very thick] plot (\x,\x)
        node[right] {IPWE};

        
        
        
    \end{tikzpicture}
    \vspace{-1em}
    \caption{PIL and IPWE}
    \label{fig:lower-bounds}
\end{wrapfigure}
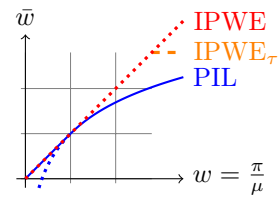

One way to reduce the large IPWE variance is to clip its large weights by replacing $w$ with $\bar{w}_\tau=\min(w, \tau)$ for a reasonable threshold $\tau>0$. Here, $\tau$ is a bias-variance trade-off parameter.
Fixing $\tau$, Clipped-IPWE is
\begin{equation}
    \max_\pi \Delta {\rm IPWE}_\tau(\pi) = \frac{1}{n}\sum\nolimits_{i=1}^n(\bar{w}_{\tau,i}-1)r_i. \label{eq:ipwe-tau}
\end{equation}

Instead, in offline learning, we take a different perspective on IPWE$_\tau$.
Assuming that the rewards are nonnegative, $\bar{w}_\tau\leq w$ lets IPWE$_\tau$ to be always a lower-bound on IPWE.
Maximizing lower bounds as a surrogate objective is common practice. Generalizing this observation, we can arrive at many extensions of IPWE. One that we focus on is what we call the policy improvement lower bound estimator (PIL), which is based on the inequality $\log(w)\leq w-1, \forall w>0$ (Figure~\ref{fig:lower-bounds}). Depending on the logging scenario - with or without 
probabilities $\mu$, we define the following objectives:
\begin{talign}
    {\rm PIL}_{\mu}(\pi)
    &= \frac{1}{n}\sum_{i=1}^n r_i 
    \log w_i1_{\{w_i\geq1\}} + (w_i-1)1_{\{w_i<1\}}
    \nonumber
    \\
    {\rm PIL}_{\emptyset}(\pi)
    &= \frac{1}{n}\sum_{i=1}^n r_i \log w_i.
\label{eq:pil}
\end{talign}

Without the logged probabilities $\mu$, we can also approximate IPWE up to a constant value. Consider the cross-entropy loss and its minimization,
\begin{equation}
    \argmin_\pi {\rm CE}(\pi; r) = -\frac{1}{n}\sum\nolimits_{i=1}^n r_i\log\pi(a_i\mid x_i).
    \label{eq:ce-pg}
\end{equation}
We see that, up to log-separable constant terms in the parameters of the logging policy, the PIL objective is equivalent to minimizing the reward-weighted cross-entropy loss for next-action predictions. That is ${\rm CE}(\pi;r) = {\rm CE}(\mu;r) - {\rm PIL}_{\emptyset}(\pi)$.
As a result, CE may not require logged action probabilities yet enjoys the additional causality justification than other offline learning objectives, such as Bayesian personalized ranking and triplet loss. In particular, the other objectives are more likely to be biased by the logging policy and their negative sampling processes.

All these approximations of IPWE, come with an important indicator of the biases they induce - violations of the self-normalizing property. The self-normalizing property is that $\mathbb{E}_\mu(w) = \mathbb{E}(\pi-\mu)=1-1=0$.
When $w$ is replaced with $\bar{w}$, the violation is empirically
\begin{equation}
  {\rm Gap} = \frac{1}{n}\sum\nolimits_{i=1}^n (1-\bar{w}_i),
  \label{eq:gap}
\end{equation}
where $\bar{w}$ generalizes to any valid lower-bound surrogates. The theorem below shows the relationship between this observable quantity and the unobserved sub-optimality due to the use of a surrogate objective.
\begin{restatable}[Probability gap]{thm}{thmGap}
    \label{thm:gap}
    For any $\bar w\leq w$, assuming $0\leq r\leq R$,
    the approximation gap can be bounded by the probability gap, in expectation:
    \begin{equation}
        0
        \leq
        \mathbb{E}_\mu[(w-\bar{w})r]
        \leq
        \mathbb{E}_\mu {[\rm Gap]} R.
    \end{equation}
\end{restatable}

In particular, when $\bar w = 1+\log w$, Gap has a simple form $-\frac{1}{n}\sum_{i=1}^n\log w_i$, which equals to one of the IML objectives that we introduce next.

\subsection{Policy imitation for variance estimation}

Another way to reduce IPWE variance is by adding regularization terms, that penalize the variance of the estimated rewards of the new policy. However, direct IPWE variance estimation \citep{swaminathan2015counterfactual} is problematic, because it requires the unreliable IPWE mean estimation in the first place. To this end, we propose to use policy imitation learning (IML) to bound the IPWE variance.

We define IML by empirically estimating the Kullback-Leibler (KL) divergence between the logging and the proposed policies,
${\rm KL}(\mu \| \pi) = \mathbb{E}_\mu\log\frac{\mu}{\pi}
= -\mathbb{E}_\mu\log w$.
We consider three logging scenarios - full logging where we have access to the logged probabilities of all actions, partial logging where we only know the logged probability of the taken action and missing where no logging probabilities are available. Depending on the amount of logging, our definition of IML has the following forms:
\begin{talign}
    &{\rm IML}_{\rm full} (\pi)
    =
    - \frac{1}{n}\sum_{i=1}^n \sum_{\vct a\in\mathcal{A}(x_i)} 
    	\mu(\vct a | \vct x_i)\log
        w(\vct a | \vct x_i);
    \nonumber
    \\
    &{\rm IML}_{\rm part} (\pi) 
    =
    - \frac{1}{n}\sum_{i=1}^n \log w_i;
    \label{eq:iml-kl}
    \\
    &{\rm IML}_{\rm miss} (\pi) 
    =
    - \frac{1}{n}\sum_{i=1}^n \log \pi(a_i | x_i) - {\rm CE}(\mu;1),
    \nonumber
\end{talign}
where CE$(\mu;1) = - \frac{1}{n}\sum_{i=1}^n \log \mu_i$ is a log-separable constant term similar to \eqref{eq:ce-pg}, but without the reward weighting. The following theorem shows that IML is a reasonable surrogate for the IPWE variance.

\begin{restatable}[IML and IPWE variance]{thm}{thmIMLvariance}
\label{thm:iml-and-ipwe-variance}
Suppose $0\leq r\leq R$ and a bounded second-order Taylor residual $\lvert -\mathbb{E}_\mu\log w  - \mathbb{V}_\mu (w-1)  \rvert \leq B$, the IML objective is closely connected to the $\Delta$IPWE variance
  \begin{equation}
  \mathbb{V}(\Delta {\rm IPWE})
  \leq
  \frac{1}{n}
  \Bigl(
  2\mathbb{E}_\mu\bigl({\rm IML}\bigr)+B
  \Bigr) R^2.
  \end{equation}
\end{restatable}

Proof by Taylor expansion around $w=1$, $\mathbb{E}_\mu ({\rm IML}) = -\mathbb{E}_\mu \log (w) \approx \frac{1}{2}\mathbb{E}_\mu(w-1)^2$, where the first-order approximation term is exactly $\mathbb{E}_\mu (w-1)=\mathbb{E}(\pi-\mu)=0$.

\begin{corollary}
These two imitation methods are second-order similar: $\min_\pi-\mathbb{E}_\mu \log (\pi)\approx \min_\pi \frac{1}{2}\mathbb{E}_\mu\bigl(\frac{\pi}{\mu}-1\bigr)^2$.
\end{corollary}



\subsection{Other Properties}
\label{sec:other}
\textbf{Generalizability:} In stochastic bandits, optimal policies are often greedy.
While classical arguments suggest that IPWE can learn greedy policies unbiasedly, e.g., with a saturated softmax, such policies may not generalize well.
For example, a greedy policy for binary selections may look like $\pi(y\mid x)=\frac{e^{y\theta^\top x}}{1+e^{\theta^\top x}}, \forall y\in\{0,1\}$, which saturates with $\|\theta\|_2\to\infty$.
Unfortunately, learning models with large weights often suggests large model complexity, which easily leads to overfitting \citep{bartlett1999generalization}.

Instead, Bayesian decision theory suggests a two-step approach: (1) fit a (reward-posterior) probability distribution of the optimal actions, which leads to smoother functions and (2) apply greedy argmax (or probability sharpening to handle ties).
Along this line, CE uses a log-softmax link function, which is a convex, margin-based loss function that further improves generalization.
It yielded better empirical results in Section~\ref{sec:criteo} and further motivates \eqref{eq:ce} as an offline learning objective.

\textbf{Adaptivity to unknown $\mu$.}
The fact that the CE objective \eqref{eq:ce} is independent to the logging policy $\mu$ indicates that we can optimize a causal objective without knowing or needing to estimate the underlying propensities, which avoids the potential pitfalls in model misspecification and confounding variables.
The following theorem establishes that optimizing $CE(\pi; r)$ based on the observed samples is implicitly maximizing the lower bound $\E_{\mu}r\log(\pi^*/\mu)$ for an unknown $\mu$.
\begin{restatable}[Statistical learning bound]{thm}{thmAdaptivity}
Let $\mu$ be the unknown randomized logging policy and $\Pi$ be a policy class. 
Let $\pi^* = \argmax_{\pi\in \Pi} {\rm CE}(\pi; r)$.
Then with probability $1-\delta$, $\pi^*$ obeys that
\begin{align*}
\E_{\pi^*} r - \E_{\mu} r \geq 
\E_{\mu}r\log(\nicefrac{\pi^*}{\mu}) \geq  \max_{\pi\in\Pi}
\bigl\{\E_{\mu}r\log\bigl(\nicefrac{\pi}{\mu}\bigr)\bigr\}\\
- O\left(\frac{\log(\max_{\pi\in\Pi}D_{\chi^2}(\mu\|\pi))+\log(|\Pi|/\delta)}{\sqrt{n}}\right)
\end{align*}
\end{restatable}
 Please refer to Appendix~\ref{app:concentration} for proofs and discussions.

A caveat is that although we can optimize the lower bound, we cannot explicitly evaluate the resulting lower bound of the policy improvement (e.g., to tell whether it is positive), without knowing $\mu$.
A heuristic solution is to use $\hat{\mu} = \argmin_\pi \mathrm{IML}(\pi)$ as a surrogate of $\mu$.




\begin{restatable}[Connections to natural policy gradients (NPGs) \citep{kakade2002natural}]{lemma}{lemNPG}
  Suppose the policy class is parametrized by $\vct \theta$, differentiable, and of the form $\pi(\vct a\mid\vct x;\vct \theta)$.
  Suppose the logging policy also resides in the policy class, as $\mu(\vct a\mid\vct x)=\pi(\vct a\mid\vct x;\vct\theta_0)$.
  The constrained optimization problem of natural policy gradient is a linear approximation to the PIL-IML in Lagrangian function form:
  \begin{align}
    \argmax_{\Delta \vct\theta}
    &\;
    \mathbb{E}_\mu
    \biggl[
      r(\vct x,\vct a)
      \biggl(
      \frac{
        \pi(\vct a\mid\vct x;\vct\theta_0+\Delta\vct\theta)
        }{
        \mu(\vct a\mid\vct x)
        }-1
      \biggr)
    \biggr]
    \label{eq:npg}
    \\
    {\rm s.t.}
    &\;
    \mathbb{E}
    \Bigl(
      {\rm KL}
      \bigl(
      \mu(\vct a\mid\vct x)
      \,\|\,
      \pi(\vct a\mid\vct x;\vct\theta_0+\Delta\vct\theta)
      \bigr)
      \Bigr)
    \leq\epsilon^2.
    \nonumber
  \end{align}
\end{restatable}

While PIL-IML can connect to NPG when the logging policy is included in the policy class, we should notice that the scope of PIL-IML is more general.
In offline learning, we also consider problems where the logging policy may not be realizable in the policy class.
In these problems, PIL-IML is still a valid objective, whereas NPG may not be properly evaluated.

Joachims et al., \citep{joachims2018deep} showed some empirical successes using a Lagrangian form of a ``self-normalized'' SNIPS estimator, but did not provide much justification.
We can show that the Lagrangian formulation of SNIPS transforms the original objective to a conservative one similar to our PIL-IML. To see this, notice that the optimal Lagrangian multipliers in \citep{joachims2018deep} are always around $1$, which equivalently means that the rewards are non-negative, agreeing with our intuitions.


\textbf{Doubly-robust estimators} can be natural extensions for PIL-IML. Please see Appendix~\ref{sec:dr} for details.

\section{IML for causal exploration}
In addition to variance control, IML has a number of useful properties and applications.

\textbf{IML causality diagnosis:}
We define the IML training loss to be the objective values of \eqref{eq:iml-kl} given partial or full logging probabilities.
A positive IML training loss indicates logging biases such as confounding variables, likely due to the exclusion of engineered features that existed in the original logging policy, or policy misspecification  when the true logging policy is not in the learned class. This property extends the classical propensity fitting methods \citep{robins1994estimation,robins1995semiparametric} that impute missing probabilities, 
which do not often check the feasiblity of the imputations when $\mu$ is assumed to be unknown.

\begin{restatable}[IML diagnosis]{lemma}{lemIMLdiagnosis}
\label{lem:IMLdiagnosis}
    Suppose the model family does not contain the logging policy $\Pi\not\ni\mu$, then $\min_{\pi\in\Pi}\mathbb{E}{\rm KL}(\mu\|\pi)\geq 0$.
    For example, if $\mu$ is a policy based on variables $x=(x_1,x_2)$, yet $\Pi$ contains policies with only support on $x_1$, then 
    $\min_{\pi\in\Pi}\mathbb{E}{\rm KL}(\mu\|\pi)\geq \mathbb{E}I_\mu(a;x_2\mid x_1)\geq0$,
    where $I_\mu$ is the mutual information 
    between the logging policy and the confouding variable.
    Equality is found at $\pi(a\mid x_1)=\mathbb{E}[\mu(a\mid x)\mid x_1],\forall x_1\forall a$.
\end{restatable}

We often measure the IML feasibility gaps - i.e. indicators of how ``far'' the true logging policy is from the class of policies we are using - by perplexity (PPL), which is the exponent of the original IML loss. A perplexity of $K$ implies that even after finding the policy in our class which best fits the logged data, the remaining uncertainty in the original action policy is equivalent to a uniform selection among $K$ candidates. Perfect IML-fitting implies a CE of zero and a PPL of one.

\textbf{IML for pure exploration: }In batch offline cases, where we have multiple sequential opportunities to interact with the world - we suggest data recollection through online application of IML policies in \eqref{eq:iml-kl}.
There are three benefits: the variance of each IML policy is small, due to the connection between IML objective and IPWE variance (Theorem~\ref{thm:iml-and-ipwe-variance}); the performance of IML policy is predictable in offline evaluation and is typically comparable with the logging policy; lastly, with positive training loss and comparable performance, IML-resampling may greatly reduce model complexity by removing unexplained but unimportant decision factors from the logging policy. As a result, the new policies tend to be more exploratory. The improvements can be measured by increase in the entropy of the policy, which we quantify below (Proof in Appendix~\ref{sec:entropy-increase-proof} is nontrivial due to action-induced distribution shifts).

\begin{restatable}[Entropy increase]{thm}{thmEntropy}
\label{thm:entropy}
	Let $\vct x=(\vct x_1,\vct x_2)^\top$ be the vector of observed and confounding variables, respectively.
    If $\pi$ is the marginalization of the logging policy,
    $\pi(a\mid x_1)=\mathbb{E}[\mu(a\mid x)\mid x_1], \forall x_1\forall a$,
    we may guarantee an increase of expected entropy than that of the logging policy:
    \begin{equation}
    \mathbb{E} H(\pi) - \mathbb{E} H(\mu)
    = \mathbb{E} \bigl( {\rm KL}(\mu\|\pi) \bigr)
    \geq 0.
    \end{equation}
\end{restatable}



\section{Simpson's paradox and simulations}
Simpson's paradox \citep{simpson1951interpretation,julious1994confounding,bottou2013counterfactual} is often used to explain the importance to remove confounders when modeling rewards.
However, we use the example differently; 
we simulate action randomization that also leads to correct action recommendations.
Further, we validate the IML theoretical properties that detect confounders and improve exploration.

\begin{table}[ht]
    \centering
    \caption{Simpson's paradox for kidney stone treatments}
    \label{tab:kidneystone}
    \begin{tabular}{lll}
    Context &Open surgery& Small puncture\\\hline
    Small & \textbf{93\% (81/87)}  & 87\% (234/270) \\
    Large & \textbf{73\% (192/263)} & 69\% (55/80)\\
    Hidden & 78\% (273/350)& \textbf{83\% (289/350)} \\\hline
    \end{tabular}
\end{table}

In the Simpson's paradox example, a kidney stone treatment dataset with two actions is presented: an open surgery treatment or a small puncture treatment. The dataset was collected with an implicit bias where most people with large stones were treated with open surgery and most with small stones were treated with punctures, due to medical practices (such as risks and recovery times, which we do not model). 
An absolute majority of patients have small stones, which have higher cure rates with either treatment. As a result, despite the fact that an open surgery had a higher cure rate with either size of stones, regression on the treatment type without knowledge of the stone size would lead to the false conclusion that small punctures are correlated with higher cure rates (Table~\ref{tab:kidneystone}). It would seem that accurate reward modeling based action selection is impossible without the stone size contexts.


\begin{table}[ht]
\centering
    \caption{Offline learning with logged probabilities.}
    \label{tab:simulate-logging}
    \vspace{0.2em}
  \begin{tabular}{clcrc}
  \toprule
  Context & Action & Probability &  Size &  Cure \\\midrule
  Hidden & Surgery  &  24\% &  87 &  93\%  \\
  Hidden & Puncture &  76\% & 270 &  87\%  \\
  Hidden & Surgery  &  77\% & 263 &  73\%  \\
  Hidden & Puncture &  23\% &  80 &  69\%  \\
  \bottomrule
  \end{tabular} 
\end{table}

On the other hand, both treatment actions have nonzero frequencies given any stone size contexts.
We could alternatively assume that the actions are randomized with their probabilities logged, given the hidden contexts (Table~\ref{tab:simulate-logging}).
Note that these probabilities are conditional on the internal states of the decision system and are different from the observed marginal probabilities, e.g.,
$
\mu(\text{Surgery} \mid \text{Small Stone}) = \frac{87}{87+270}   \approx 24\%, 
\mu(\text{Surgery} \mid \text{Large Stone}) = \frac{263}{263+80}  \approx 77\%.
$

In this way, unbiased offline learning is possible by weighting the action effects according to their inverse action probabilities, i.e., via IPWE.
Thus logging would lead to the correct decision (surgery treatment). Doubly robust (DR) estimators may further reduce IPWE variance, but the amount of reduction depends on the quality of the reward model and can even be negative in some cases (Table~\ref{tab:offline-simulated} first column).

\begin{table}[ht]
\centering
\caption{Estimated cure rates of the surgery treatment}
\label{tab:offline-simulated}
\begin{tabular}{lcc}
\toprule
Method & Original data & IML-resampled \\\midrule
IPWE & 83.3$\pm$5.0 & 83.3$\pm$3.7 \\
Q-learning & 78.0$\pm$1.6 & {\bf 83.3$\pm$1.4} \\
DR & {\bf 83.3$\pm$2.6} & 83.3$\pm$2.5 \\
DR worst case & 83.3$\pm$5.4 & 83.3$\pm$4.0\\
\bottomrule
\end{tabular}
\end{table}

\begin{figure*}[t]
  \centering
  \begin{subfigure}[b]{0.305\linewidth}
    \centering\includegraphics[height=0.95\linewidth]{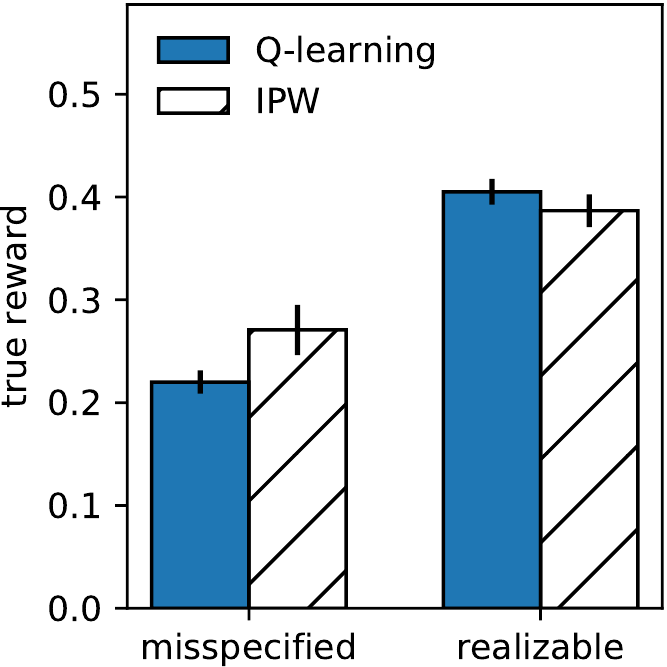}
    \caption{Unbiased IPWE is better than Q-learning with misspecified models.}
    \label{fig:optdigits-bias}
  \end{subfigure}\hspace{1.5em}%
  \begin{subfigure}[b]{0.305\linewidth}
    \centering\includegraphics[height=0.95\linewidth]{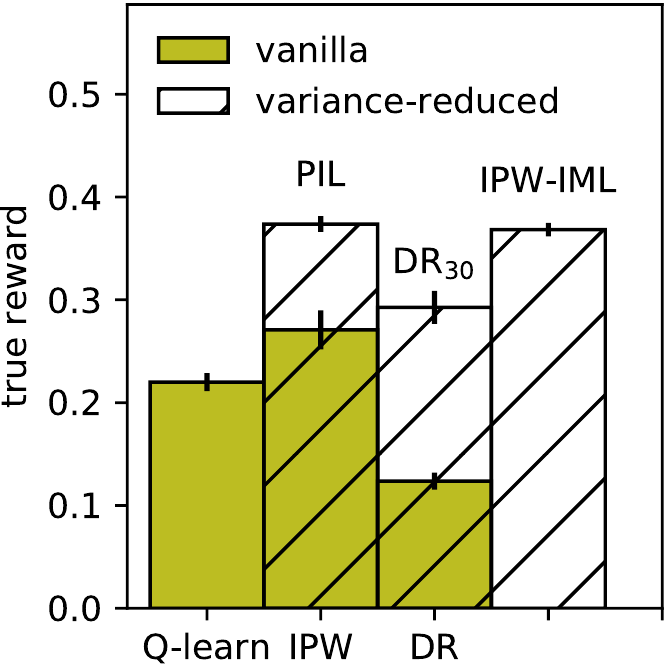}
    \caption{Variance reduction techniques further improve offline learning.}
    \label{fig:optdigits-variance}
  \end{subfigure}\hspace{1.5em}%
  \begin{subfigure}[b]{0.305\linewidth}
    \centering\includegraphics[height=0.95\linewidth]{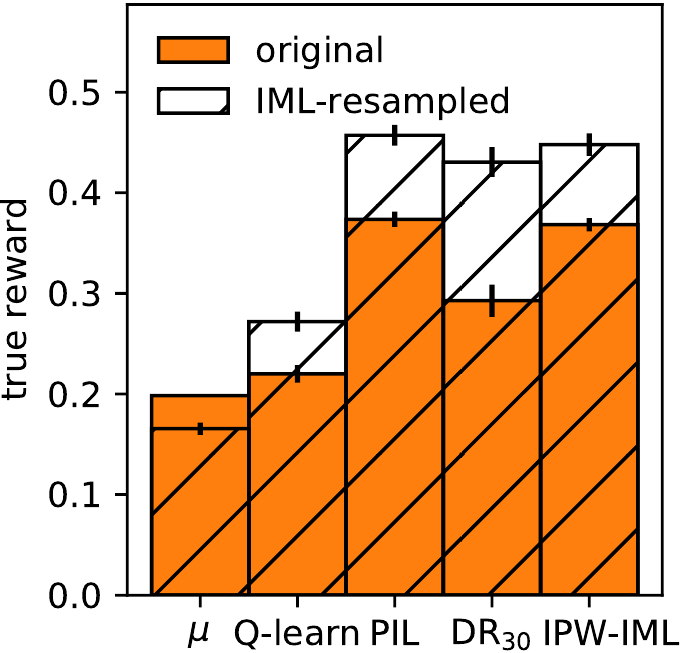}
    \caption{Online application of IML improves future offline learning.}
    \label{fig:optdigits-resample}
  \end{subfigure}
  \caption{Multiclass-to-bandit conversion on UCI optdigits dataset. 
  Proposed improvements are in hollow style. Results for the other UCI datasets are included in the Appendix~\ref{app:uci}.}
  \vspace{-1em}
  \label{fig:uci}
\end{figure*}

Based on our results, we use IML to first examine the logged action probabilities.
In this case, IML is underfitted with $1.15 (>1)$ perplexity with a uniform policy, which suggests that there exist unobserved confounders.
Since we do not have access to the additional confounders, we could simply resample data with the IML-fitted policy.
We simulate resampling by weighing the examples according to the ratios between the IML policy and the logged probabilities.
Due to self-normalization properties of the weights, the effective sample size (sum of all weights) remains the same for both small and large stone cases for fair comparisons.
Table~\ref{tab:offline-simulated} shows that IML-resampling decreases the variance of all methods. 
This is because the new logging probabilities become more balanced (uniformly random trials), without depending on the hidden decision variable: the unknown stone sizes.

\section{UCI bandit simulations}
We use UCI multiclass-to-bandit conversion datasets that originally appeared in \citep{dudik2011doubly,wang2016optimal} to simulate contextual bandit problems.
For each data point, we sampled one class as the action and observed partial feedback whether the sampled class is the true class for that data point. Following previous literature, we constructed the logging policies by the softmax prediction of a linear logistic regression classifier trained on skewed datasets with induced covariate shifts.
Since we have full knowledge of the original multiclass labels, we can exactly evaluate the learned policies during test time by multiclass fractional accuracy.
We used 50\% train-test splits where the training sets were converted to bandit datasets.
Figure~\ref{fig:uci} also reports $95\%$ confidence intervals from $100$ repetitions.
With this dataset, we show how reward modeling biases and IPWE variances affect offline learning, how to reduce variance using PIL-IML, and how to adapt IML into a batch-online method to collect better data for future offline learning.

As discussed earlier, \textbf{Q-learning biases} can come from missing confounders and/or model underfitting, leading to variable under-utilization.
We simulate this effect by using a second-order model,
$
\phi(\vct x,\vct a) = \vct x^\top 
UV^\top\vct a,
$
with insufficient rank. For UCI optdigits, a rank-2 model could only realize $67\%$ multiclass accuracy when trained with full information, compared with $95\%$ accuracy for a full-rank model, i.e.
$
\phi(\vct x,\vct a) = \vct x^\top 
W\vct a.
$
We call rank-2 models misspecified and full-rank models realizable.

Figure~\ref{fig:optdigits-bias} shows that Q-learning and IPWE policy learning behave differently for misspecified model families.
This is because Q-learning studies the biased correlation effects between the rewards and context-action pairs, whereas IPWE studies the unbiased causal effects of the actions given the contexts.
Therefore, IPWE lead to better actions.
On the other hand, Q-learning and IPWE policy learning behaved similarly for realizable model families, as expected from our analysis. 

\textbf{Variance-reduced methods} further improved offline learning. Figure~\ref{fig:optdigits-variance} continues the experiments with misspecified models, where the solid boxes are carried over from Figure~\ref{fig:optdigits-bias}, with the addition of the logging policy itself ($\mu$), and doubly robust (DR).

We compare three different variance-reduction approaches: PIL, PIL with DR extensions, and the original IPWE with IML regularization.
All three approaches improved the final policy.
The results were not very sensitive to the parameter choices, which we picked $\epsilon=10^{-4}$, after a coarse grid search.

\textbf{IML causality diagnosis} was able to detect model underfitting due to insufficient rank. With a theoretical limit of $0$, the rank-2 model family could only achieve $0.60$ training loss, which indicates that the best IML policy cannot explain $\exp(0.60)=1.82$ perplexity in the logged actions. Full-rank action policies can achieve a near-zero ($0.02$) training loss.

\textbf{IML-resampling to collect additional data} improves the policies learned with all methods (Figure~\ref{fig:optdigits-resample}, changes from solid to hollow boxes), despite a small cost during IML resampling ($\mu$ as a policy in the first box).
This is because better exploration leads to smaller inverse probability weights.
Besides, IML-fitting alleviates model underfitting biases in the new data.
Finally, the cost of IML resampling can be estimated prior to applying IML online and is fundamentally unavoidable for all methods that use the same model class.

\textbf{Additional results} on the other UCI datasets are in the appendix.
In those examples, we further observed that improvements from variance reduction are significant only when IML loss is above zero. IML loss at zero indicates that there were no confounding variables or model misspecification (e.g., Figure~\ref{fig:optdigits-bias} full-rank model); and that both naive Q-learning and IPWE would perform similarly to the variance-reduced methods.

\section{Large-scale experiments}\label{sec:criteo}
We extend our study on Criteo counterfactual-analysis dataset \citep{lefortier2016large}.
This dataset is particularly interesting,
because the logging policy is in fact unrealizable from the published features and models.
We made novel discoveries on (1) the existence of large variance due to Cauchy-like importance weight distribution and (2) the existence of modeling biases and confounding variables.
We communicated and confirmed our hypotheses with the original authors.

\textbf{The dataset} contains logs of display advertisements shown to users, the hidden action probabilities, the user context features as well as features for every candidate action.
However, we observed some discrepancies when we reran the provided scripts
(Table~\ref{tab:reproduce} in Appendix~\ref{app:criteo}).

First, we noticed that the importance weight distributions, e.g., between the uniform policy and the logging policy, resemble heavy-tail distributions with \textbf{unbounded variance}. I.e., $P(|W|>w)$ decays slower than $O(1/w^2)$ in Figure~\ref{fig:IPWE-weight}.
This fact invalidates the confidence interval (CI) estimation in the original paper \citep{lefortier2016large} based on Central Limit Theorem (CLT), which requires bounded variance.
For example, while the IPWE of the uniform distribution is $44.7\pm2.1$ in the test split, it becomes $52.6\pm 18.7$ in the training split, due to one observation with importance weight $4.9\times 10^4$.

To reduce heavy-tail uncertainties, we not only used weight clipping \citep{bottou2013counterfactual,wang2016optimal}, but also novelly applied \textbf{subsampling bootstrap} \citep{politis1994large},
which only assumes that $n^{-\beta}\bigl({\rm IPWE}_n - \mathbb{E}({\rm IPWE})\bigr)\xrightarrow[{\rm dist.}]{n\to\infty}F$ asymptotically converges to any fixed distribution $F$,
where $\beta$ can be fitted from data.
By weight clipping, we showed that $\beta$ improved from $0.3$ to near $0.5$, which would be the CLT ideal case.
See Appendix~\ref{app:bootstrap} for more details.

Then, using IML \eqref{eq:iml-kl}, we estimated the KL-divergence between the logging policy and its realizable imitation to be 0.40($\gg0$). I.e., the logging policy is \textbf{unrealizable} by an exponential-family model with raw features, which contradicts \citep{lefortier2016large} and essentially implies confounders (Lemma~\ref{lem:IMLdiagnosis}).
Even with a more complex second-order model with 256-dimensional embedding, IML improved to $0.35$, but was still large.
While having perfect imitation is a sufficient but not necessary condition,
as long as key decision variables are included \citep{strehl2010learning},
we found the situation practically difficult.

\begin{table}[t]
\centering
\caption{Criteo counterfactual analysis dataset \citep{lefortier2016large}.
}
\label{tab:criteo}
\begin{tabular}{l c c c}
\toprule
\makecell{Approach \\ $\lceil\rceil$=greedy} &
\makecell{Offline Est. \\ ($\times 10^4$)} &
\makecell{Gap \eqref{eq:gap} \\ (100\%)} &
\makecell{Paired $\hat\Delta$ \\  ($\times 10^4$)} \\
\midrule
{Logging} & (53.3, 53.7) & ( 0.0, 0.0) & ( 0.1, 1.8) \\
{IML}     & (51.5, 53.3) & (-0.3, 0.3) & ( 0.0, 0.0) \\
{Uniform} & (41.8, 52.6) & ( 7.0, 8.0) & (-10\phantom{.}, 0.1) \\
{$\lceil$Q-learn$\rceil$}
          & (49.3, 55.9) & ( 3.0, 4.0) & (-2.8, 3.1) \\
{POEM \citep{swaminathan2015counterfactual}}
          & (51.4, 53.7) & ( 0.1, 0.7) & (-1.0, 1.1)\\
{IPWE$_{100}$}
          & (51.9, 54.5) & (-0.2, 0.5) & (-0.6, 1.9) \\
{PIL-IML} & (52.3, 53.7) & (-0.2, 0.2) & ( 0.0, 0.8) \\
{$\lceil$IML$\rceil$}
          & (53.0, 55.1) & (-0.4, 0.2) & ( 0.2, 2.4) \\
{$\lceil$PIL-IML$\rceil$}
          & (53.1, 55.2) & (-0.3, 0.3) & ( 0.6, 2.9) \\
\bottomrule
\end{tabular}
\vspace{-1em}
\end{table}

Table~\ref{tab:criteo} reports offline 95\% CIs using IPWE with weights clipped at 500 and subsampling bootstrap.
Similar to \citep{bottou2013counterfactual}, the first column includes both IPWE uncertainties and any missing clicks from the self-normalization Gaps, which we report in the second column.
Since the global click rate is very low (around 5\%),
using $0<R<1$ extreme values may overestimate the upper bounds.
Reasonably, we used the additional assumption that the expected click rates is always between $0\sim1\%$ in any reasonable subsets, twich their global average.
The last column reports any improvements $\hat\Delta$ compared with the realizable IML.
Here, paired-tests were used to avoid reward estimation noise.
The formula also used weight-clipping and Gap-filling combinations: 
$\hat\Delta = \frac{1}{n}\sum_{i=1}^n \Bigl[{\rm clip}(\frac{\pi_i-\hat\mu_i}{\mu_i}, -\tau, \tau)r_i\Bigr] +
\Bigl[\frac{1}{n}\sum_{i=1}^n {\rm clip}(\frac{\pi_i-\hat\mu_i}{\mu_i}, -\tau, \tau)\Bigr]R, {\rm \; s.t. \;} 0<R<1\%$.

Beating the realizable IML are the logging policy with its secret features, ${\rm PIL}_\mu + 0.8 \,{\rm IML}$ with a tight margin, and most interestingly, greedy policies by sharpening PIL-IML or even IML, which is reward-agnostic but near-optimal.
Intuitively,
optimal policies are often greedy in stochastic environments.
An option is always desirable no matter how small its improvements are, as long as they are consistent.
While IPWE$_{100}$ would also yield greedy solutions, 
they tend to learn saturated softmax, which may not generalize well
(Section~\ref{sec:other}).


Notice, this stochastic view may have taken advantage of the temporal overlaps between the train/test splits.
In contrast, randomized policies are more popular in adversarial bandits to account for temporal uncertainties \citep{bubeck2012regret,mcmahan2013ad}.
Similarly, if we were allowed to collect new data, we would use PIL-IML (non-greedy) to improve exploration (Theorem~\ref{thm:entropy}).
The perplexity would increase from $3.6$ to $5.2$ and the heavy-tail situations would alleviate.

\section{Discussion and conclusion}
\emph{Why should I imitate a policy when I already have the logged propensities?}
First, IPWE might suffer from high variance and DR is often not much better because we seldom observe all major variables that contribute to the variance of the reward. Second, whether the imitated policy produces probabilities that match the logged propensities or not reveals whether there are hidden decision variables used by the logging policy. 

\emph{What can I use an imitated policy for?}
We can run it to collect new data, which guarantees that (if we cannot discover or log them for some reasons) the new data will have a realizable policy without unknown confounding decision variables. Besides, the IPWE estimate of the imitated policy should not suffer from high variance and we will have good evidence whether the imitated policy performs similarly to the unrealizable logging policy in most cases. 

\emph{What are the take-home messages for data scientists and developers?}
The most important message is that one should be cautious about using and evaluating Q-learning approaches in problems where decisions are involved. We highlight the value of having randomized policies that allow one to marginalize over unobserved confounding variables and make statistically valid inferences despite possible confounders. Lastly, the probabilities corresponding to actions that are not taken are also useful and can be used to reduce variance in offline policy valuation and policy optimization.

%% file: ma2018appendix.tex



\section{Related work}

Clipped-IPWE was originally used in offline evaluation, which, however, leaves the remaining gaps to be estimated. \cite{bottou2013counterfactual} estimated the remainders via Bernstein's inequality, which may empirically be a loose estimation, and \cite{wang2016optimal} used reward modeling, which assumes nonexistence of confounding variables and universal approximators, which are under our scrutiny.

Some theoretical justification for the large variance in IPWE is provided in \citep{cortes2010learning,wang2016optimal}.
Doubly robust estimation (DR)  was initially proposed to
take advantage of a possibly inaccurate Q-learning model for estimating the expected rewards
\citep{robins1994estimation,dudik2011doubly,jiang2015doubly}, 
yet DR does not fundamentally solve the large variance inverse probabilities \citep{kang2007demystifying}.
Another variant is to clip IPWE into a head average and a tail bound
\citep{bottou2013counterfactual,bembom2008data},
but this approach has not been formulated as an stochastic optimization objective.
As an extension, SWITCH method \cite{wang2016optimal} was proposed to tighten the mean-squared error (MSE) bounds in clipped IPWE via Q-learning estimates to achieve minimax-optimality. However, SWITCH was not directly used for optimization, either.
%
%
On policy imitation, a related topic is to learn the probabilities through propensity fitting \citep{strehl2010learning,austin2011introduction}. We use similar techniques, but only for regularizing safe exploration, where we also assume knowledge of the exact logged probabilities. 
Our solution connects to other variance reduction techniques \citep{swaminathan2015counterfactual,joachims2018deep} and we show that ours are more stable with varying hyperparameter choices.

In the RL literature, Munos et al \citep{munos2016safe} uses importance weighting to improve Q-learning convergence. When choosing the importance weighting as a soft-Q policy \citep{haarnoja2017reinforcement}, they proved a contraction property that leads to final convergence to true Q*.
While the method works well for tabular settings, it is unclear whether Q* is fundamentally easy to estimate in continuous settings from data with confounding variables.

Similarly, trust-region policy gradient method \citep{schulman2015trust} used Pinsker's inequality to motivate a worst-case KL-divergence as regularization. Differently, we directly motivated expected KL-divergence as regularization.
This was indeed their empirical choice as well, but with less clear intuitions.
The follow-up proximal policy gradient method \citep{schulman2017proximal} used direct IPWE with heavy clipping on both sides. Our method leads to a smoother policy that tend to generalize better (Section~\ref{sec:other}).

\section{Example on IPWE large variance and statistical analysis}

\begin{example}[Epsilon-greedy]
\label{eg:epsilon-greedy}
Suppose the logging policy (ignores the context and) tosses a biased coin to choose between two actions: a rare action $\mu(A)=\epsilon\ll 1$ and a default action $\mu(C)=1-\epsilon$.
Suppose the rewards are noiseless and only given to the rare action, $r(a)=1_{\{a=A\}}$.
With $n$ examples,
the rare action will be included in the logging dataset at least once with high probability $(1-(1-\epsilon)^n)\geq 1-e^{-n\epsilon}$.

An (optimal) deterministic policy that always chooses $A$ has expected reward $1$. While IPWE is unbiased, its mean-squared error (MSE) can be at least $\frac{1}{n}\mathbb{E}_\mu [(\frac{\pi}{\mu}r)^2-(\mathbb{E}_\pi r)^2] = \frac{1}{n}[\epsilon(\frac{1}{\epsilon})^2-1] \geq \frac{1}{2n\epsilon}$, because the rare event has significant weight $\frac{1}{\epsilon}$.
On the other hand, Q-learning has zero variance but a unit bias if $A$ is not observed, with probability at least $e^{-n\epsilon}$.
Comparatively, the MSE in Q-learning is exponentially smaller than IPWE in fully observed but imbalanced datasets.
\end{example}

\section{Concentration inequalities for PIL-IML}\label{app:concentration}
The concentration for the weight clipping-based ${\rm PIL}$ straightforwardly follows from the boundedness and Hoeffding's inequality, as was observed in \citep{bottou2013counterfactual}.

In this section, we derive the concentration inequality for the cross-entropy based ${\rm PIL}$.
\begin{lemma}[Exponential tails for the cross-entropy weights]\label{lem:concentration}
Let $\log(\pi(x,a)/\mu(x,a))$ be a random variable induced by $x\sim \cD, a|x\sim \mu(a|x)$.
For all $t>0$
  $$
\P\left[ |\log\frac{\pi(x,a)}{\mu(x,a)}| \geq t\right]  \leq  e^{-t}(1 +  \E_\pi[(\mu/\pi)^2])
$$
where $\E_\pi[(\mu/\pi)^2] =: D_{\chi^2}(\mu\|\pi)$ is the $\chi^2$-divergence.

Let $\bar{w}(x,a) = \begin{cases}
\log\frac{\pi(x,a)}{\mu(x,a)} & \mbox{ if }\pi>=\mu;\\
\frac{\pi(x,a)}{\mu(x,a)}-1 & \mbox{otherwise.}
\end{cases}$
then 
$$
\P\left[ |\log\frac{\pi(x,a)}{\mu(x,a)}| \geq t\right]  \leq  2 e^{-t}
$$
\end{lemma}
\begin{proof}
For readability, we use $\pi$ and $\mu$ as shorthands for random variables $\pi(x,a)$ and $\mu(x,a)$.
By Chernoff's argument and Markov's inequality:
\begin{align*}
    \P\left[ |\log\frac{\pi }{\mu}| \geq t\right]  \leq&  \P\left[ e^{|\log\frac{\pi}{\mu}|} \geq e^t\right]\\
    \leq& e^{-t} \E_\mu[ e^{|\log\frac{\pi}{\mu}|} ]\\
    =&e^{-t} \left\{ \E_\mu[ e^{\log\frac{\pi}{\mu}} \mathbf{1}(\pi\geq\mu) ] +  \E_\mu[ e^{\log\frac{\mu}{\pi}} \mathbf{1}(\pi<\mu) ] \right\}\\
    =& e^{-t} \left\{ \E_\mu[ \frac{\pi}{\mu} \mathbf{1}(\pi\geq\mu) ] +  \E_\mu[ \frac{\mu}{\pi} \mathbf{1}(\pi < \mu) ]  \right\}\\
    \leq&e^{-t} \left(1 +  \E_\pi[ \frac{\mu^2}{\pi^2} ]  \right).
\end{align*}
Now for the modified cross-entropy objective, by the same argument, we have
\begin{align*}
\P\left[ |\bar{w}| \geq t\right]  \leq&  e^{-t} \left\{ \E_\mu[ e^{\log\frac{\pi}{\mu}} \mathbf{1}(\pi\geq\mu) ] +  \E_\mu[ e^{\frac{\pi}{\mu}-1} \mathbf{1}(\pi<\mu) ] \right\}\\
\leq &e^{-t} \left\{ \E_\mu[ \frac{\pi}{\mu} \mathbf{1}(\pi\geq\mu) ] +  \E_\mu[ 1 \mathbf{1}(\pi<\mu) ] \right\}\leq 2e^{-t}.
\end{align*}
\end{proof}

\thmAdaptivity*

\begin{proof}[Proof sketch]
For every $\pi\in \Pi$, we use Lemma~\ref{lem:concentration} to get that the empirical estimate converges its mean for both the objective that we optimizes over and the empirical estimate of the gap. Then we apply a union bound to make it uniform over the entire policy class. Since $\pi$ is the argmax on the empirical estimates, it must also be close to the argmax on the population quantity, which concludes the proof.
\end{proof}
The assumption on discrete $\Pi$ is not essential. It can be replaced with a uniform convergence bound for infinite alphabet. 

Also note that the above lower bound is also something that we can hope to optimize without knowing what $\mu$ is. This motivates us to use policy immitation as a regularization term. $R$ is a conservative regularization weight.

Now, assume that $\mu\in \Pi$, also $\Pi$ is parameterized by a parameter vector $\theta$ and in addition $\pi$ is a smooth function in $\theta$. Then the standard policy gradient theorem tells us that there exists a policy in the neighborhood of $\mu$ that improves over $\mu$ provided that $\mu$ is not a local maxima. The maximizer of CE-IML objective resembles the natural policy gradient update which starts at $\mu$ and in fact, it magically get away with not need to know where to start.

To conclude the section, we note that this property of adaptivity in CE and CE-IML optimization implies that policy optimization might be an easier problem than policy evaluation, challenging the common wisdom that we need to know the objective function before we can optimize.

\section{Connections between IML and IPWE variance}

Many counterfactual analysis papers maximize the lower bounds of expected rewards, as they are more reliable to estimate \citep{swaminathan2015counterfactual,bottou2013counterfactual}.
Our intuition is similar.

However, we avoided direct optimization of the bounds like \cite{swaminathan2015counterfactual}:
\begin{equation}
\min_{\pi} {\rm IPWE}(\pi) + \alpha \sqrt{\mathbb{V}({\rm IPWE}(\pi))},
\end{equation}
which is further optimized by tunning $\alpha$ such that $\alpha=\lambda' \sqrt{\mathbb{V}(\rm IPWE(\pi))}$.
Instead, we connect the IPWE variance to the IML objective. Restate Theorem~\ref{thm:iml-and-ipwe-variance}

\thmIMLvariance*

\begin{proof}

First, we point out a key observation from the second-order Taylor expansion of the IML objective near $w=\frac{\pi}{\mu}\approx 1$, as
\begin{align}
	\mathbb{E}_\mu ({\rm IML})
    = -\mathbb{E}_\mu\log (w)
    &= -\mathbb{E}_\mu\log \bigl( 1+(w-1) \bigr)
    \\
    &= - \mathbb{E}_\mu \Bigl[
      (w-1)
      - \frac{1}{2} (w-1)^2
      + o\bigl((w-1)^2\bigr)
    \Bigr]
    \\
    &=\frac{1}{2} \mathbb{E}_\mu (w-1)^2 + o\bigl((w-1)^2\bigr),
    \label{eq:iml-var}
\end{align}
where $\lvert o\bigl((w-1)^2\bigr) \rvert\leq B$ is the residual and $\mathbb{E}_\mu w=1$ cancels the first-order term.

Then, suppose $\lvert r \rvert \leq R$, we relax the policy-related variance estimation
\begin{equation}
  \mathbb{V}_\mu\bigl((w-1)r\bigr)
  =\mathbb{E}_\mu\bigl((w-1)^2r^2\bigr)
  \leq \mathbb{E}_\mu(w-1)^2R^2
  \leq \Bigl( 2\mathbb{E}_\mu({\rm IML})+B \Bigr)R^2
\end{equation}

Finally, the variance $\Delta$IPWE has an additional $\frac{1}{n}$ term, because $\Delta$IPWE itself is a mean-value estimator.
\end{proof}

Note, the derived form is closely related to the variance of IPWE itself, by considering additional reward observation noise:
\begin{equation}
  \mathbb{E}_\mu (wr) = \mathbb{E}_\mu \bigl((w-1)r\bigr)
  + \mathbb{E}_\mu (r)
  \quad\Rightarrow\quad
  \mathbb{V}_\mu (wr)
  \leq 2\mathbb{V}_\mu \bigl((w-1)r\bigr) + 2\mathbb{V}_\mu (r).
\end{equation}

Comparatively, IML can be more robust, 
because it avoids influence from large probabilities.
On the other hand, IPWE/$\Delta$IPWE as a mean value can sometimes be unstable due to large inverse propensities $\frac{1}{\mu_i}$ when weight-clipping is not involved $(\tau=\infty)$.
Additionally, IML does not involve reward estimation and thus the regularization directions may be orthogonal to the policy improvements,
whereas empirical IPWE/$\Delta$IPWE may directly penalize improvements.

\begin{remark}
The Taylor expansion in Theorem~\ref{thm:iml-and-ipwe-variance} relies on bounded central moments.
This assumption may not be available in general offline learning scenarios, e.g., Example~\ref{eg:epsilon-greedy}.
However, the assumption is reasonable with variance-reduced approaches, e.g., Clipped-IPWE, IPWE-IML, and POEM\citep{swaminathan2015counterfactual}.
\end{remark}

\section{Connections between PIL-IML and natural policy descent}

Natural policy descent \citep{kakade2002natural} is a steepest descent optimization approach to solve an approximate policy optimization problem:
\begin{equation}
  \max_{\vct \theta} \mathbb{E}_\mu
  \bigl(
  r(\vct x,\vct a)
  \log \pi(\vct a\mid \vct x; \vct \theta)
  \bigr).
\end{equation}
It computes natural policy gradients (NPGs), which use the first and second-order derivatives of the objective,
\begin{equation}
  \Delta\vct \theta
  =\Bigl[
  \mathbb{E}_\mu
  \bigl(
  \nabla_{\vct\theta}
  \log\pi(\vct a\mid \vct x;\vct\theta)
  \bigr)
  \bigl(
  \nabla_{\vct\theta}
  \log\pi(\vct a\mid \vct x;\vct\theta)
  \bigr)^\top
  \Bigr]^{-1}
  \mathbb{E}_\mu
  \bigl(
  r(\vct x,\vct a)
  \nabla_{\vct\theta}
  \log\pi(\vct a\mid \vct x;\vct\theta)
  \bigr).
\end{equation}

The one-step NPG is also equivalent to the solution to the constrained optimization problem:
\begin{align}
  \argmax_{\Delta \vct\theta}
  &\;
  \mathbb{E}_\mu
  \bigl(
  \Delta \vct\theta^\top
  r(\vct x,\vct a)
  \nabla_{\vct\theta}
  \log\pi(\vct a\mid \vct x;\vct\theta)
  \bigr)
  \\
  {\rm s.t. }
  &\;
  \mathbb{E}_\mu
  \Bigl[
  \Delta \vct\theta^\top
  \bigl(
  \nabla_{\vct\theta}
  \log\pi(\vct a\mid \vct x;\vct\theta)
  \bigr)
  \bigl(
  \nabla_{\vct\theta}
  \log\pi(\vct a\mid \vct x;\vct\theta)
  \bigr)^\top
  \Delta \vct\theta
  \Bigr]
  \leq \epsilon^2,
\end{align}
where $\epsilon>0$ is the step size.

We can establish the following connections between PIL-IML and NPG.

\lemNPG*

\begin{proof}
  To show the equivalence of the objectives, we apply Taylor expansions and use the property
  \begin{equation}
    \nabla_{\vct\theta}
    \log\pi(\vct\theta)
    =
    \frac{
    \nabla_{\vct\theta}\pi(\vct\theta)
    }{
    \pi(\vct\theta)
    }
    \quad\Rightarrow\quad
    \nabla_{\vct\theta}\pi(\vct\theta)
    =
    \pi(\vct\theta)
    \nabla_{\vct\theta}
    \log\pi(\vct\theta).
  \end{equation}
  The policy-related term in the objective becomes 
  \begin{align}
    \biggl(
    \frac{
    \pi(\vct a\mid\vct x;\vct\theta_0+\Delta\vct\theta)
    }{
    \mu(\vct a\mid\vct x)
    }-1
    \biggr)
    &\approx
    \biggl(
    \frac{
    \mu(\vct a\mid\vct x) + 
    \Delta\vct\theta^\top
    \nabla_{\vct\theta}
    \pi(\vct a\mid\vct x;\vct\theta_0)
    }{
    \mu(\vct a\mid\vct x)
    }-1
    \biggr)
    \\
    &=
    \Delta\vct\theta^\top
    \nabla_{\vct\theta}
    \log\pi(\vct a\mid\vct x;\vct\theta_0)
    \label{eq:policy-improvement-log}
  \end{align}
  
  On the other hand, using the derivations in \eqref{eq:iml-var} and \eqref{eq:policy-improvement-log}, the constraint is equivalent to
  \begin{align}
    \mathbb{E}
    \Bigl(
      {\rm KL}
      \bigl(
      \mu(\vct a\mid\vct x)
      \,\|\,
      \pi(\vct a\mid\vct x;\vct\theta_0+\Delta\vct\theta)
      \bigr)
      \Bigr)
    &\approx
    \mathbb{E}
    \biggl(
    \frac{
    \pi(\vct a\mid\vct x;\vct\theta_0+\Delta\vct\theta)
    }{
    \mu(\vct a\mid\vct x)
    }-1
    \biggr)^2
    \\
    &=
    \mathbb{E}
    \bigl(
    \Delta\vct\theta^\top
    \nabla_{\vct\theta}
    \log\pi(\vct a\mid\vct x;\vct\theta_0)
    \bigr)^2.
  \end{align}
  Now, both the objective and the constraint are in the same form as NPG.
\end{proof}


\section{IML strictly positive due to confounding variables}

\label{sec:iml-diagnosis}
\lemIMLdiagnosis*

\begin{proof}
    Without loss of generality, assume $x_1=\emptyset$ and $x=x_2$. We can express the objective as
    \begin{equation}
        \mathbb{E}{\rm KL}(\mu\|\pi)
        = \int\! p(x) \sum_{a}\mu(a\mid x) \log \frac{\mu(a\mid x)}{\pi(a)} {\rm\, d}x
        = {\rm CE}(\pi, \mathbb{E}\mu(\cdot \mid x)) - \mathbb{E} H(\mu)
    \end{equation}
    which is maximized by $\pi(a) = \mathbb{E}\mu(a\mid x) = \int\! p(x)\mu(a\mid x) {\rm\, d}x$.
    
    Then we plug in the solution. Define $p_\mu = p(x)\mu(a\mid x)$, the solution becomes $\hat\pi(a) = \mathbb{E}\mu(a\mid x)=p_\mu(a)$ and the objective is
    \begin{equation}
        \mathbb{E}{\rm KL}(\mu\|\hat\pi)
        = \int\! p_\mu(x,a) \log \frac{p_\mu(a\mid x)}{p_\mu(a)} {\rm\, d}x{\rm\, d}a
        = \int\! p_\mu(x,a) \log \frac{p_\mu(a, x)}{p_\mu(a) p_\mu(x)} {\rm\, d}x{\rm\, d}a
        = I_\mu(a;x),
    \end{equation}
    where $I_\mu(a;x)$ is the mutual information due to the logging policy between $a$ and $x$, the variable that cannot be explained by the new policy.
\end{proof}

\section{Entropy increases with IML policies due to confounding variables}

\label{sec:entropy-increase-proof}
\thmEntropy*

\begin{proof}
Starting from the left hand side and using the provided equation for $\pi$,
\begin{align}
  \mathbb{E}H(\pi)-\mathbb{E}H(\mu)
  =\int
  \biggl[
    &-\sum_{\vct a}
      \pi(\vct a\mid \vct x_1)
      \log\pi(\vct a\mid\vct x_1)
    \\
    &+\int\! 
      \sum_{\vct a}
      \mu(\vct a\mid \vct x)
      \log\mu(\vct a\mid\vct x)
      p(\vct x_2\mid\vct x_1)
      {\rm \, d}\vct x_2
  \biggr]
  p(\vct x_1)
  {\rm \, d}\vct x_1
  \\
  =\int
  \biggl[
    &-\sum_{\vct a}\int\!
      \mu(\vct a\mid \vct x)
      p(\vct x_2\mid\vct x_1)
      {\rm \,d}\vct x_2
      \log\pi(\vct a\mid\vct x_1)
      \label{eq:replace}
    \\
    &+\int\!
      \sum_{\vct a}
      \mu(\vct a\mid \vct x)
      \log\mu(\vct a\mid\vct x)
      p(\vct x_2\mid\vct x_1)
      {\rm \, d}\vct x_2
  \biggr]
  p(\vct x_1)
  {\rm \, d}\vct x_1.
\end{align}

With the key observation that $\pi(\vct a\mid \vct x_1)=\pi(\vct a\mid \vct x)$ is independent of $(\vct x_2\mid\vct x_1)$, 
we may replace the summation and the integral in \eqref{eq:replace}, which yields
\begin{align}
  \mathbb{E}H(\pi)-\mathbb{E}H(\mu)
  &=\int
  \biggl[
    -\int\!\sum_{\vct a}
      \mu(\vct a\mid \vct x)
      \log\pi(\vct a\mid\vct x)
      p(\vct x_2\mid\vct x_1)
      {\rm \,d}\vct x_2
    \\
  &\phantom{=\int\biggl[}
    +\int\!
      \sum_{\vct a}
      \mu(\vct a\mid \vct x)
      \log\mu(\vct a\mid\vct x)
      p(\vct x_2\mid\vct x_1)
      {\rm \, d}\vct x_2
  \biggr]
  p(\vct x_1)
  {\rm \, d}\vct x_1
  \\
  &=\int
    \sum_{\vct a}
      \mu(\vct a\mid \vct x)
      \log\biggl(\frac{
        \mu(\vct a\mid\vct x)
        }{
        \pi(\vct a\mid\vct x)
        }
      \biggr)
      p(\vct x)
      {\rm \,d}\vct x
  \\
  &=\mathbb{E} \bigl({\rm KL}(\mu\|\pi)\bigr).
  \nonumber
\end{align}
\end{proof}

\section{Doubly robust estimation}
\label{sec:dr}



DR fixes the IPWE large variance without introducing biases.
Instead, it uses two equivalent ways to find the expectation 
$\mathbb{E}_\pi[\hat{f}\mid x] = \mathbb{E}_\mu[\frac{\pi}{\mu}\hat f \mid x]=\sum_a \hat f(x,a)$
of a known function $\hat{f}$ that can be evaluated for any action, as
\begin{equation}
\max_\pi {\rm DR}(\pi) 
	= \frac{1}{n}
	\sum_{i=1}^n
	\frac{\pi(\vct a\mid\vct x_i)}{\mu_i}
    \bigl( r_i-\hat{f}(\vct x_i,\vct a_i) \bigr) 
    + \frac{1}{n}\sum_{i=1}^n \sum_{\vct a\in\mathcal{A}(\vct x_i)} \pi(\vct a\mid\vct x_i)\hat{f}(\vct x_i,\vct a).
    \label{eq:dr}
\end{equation}

DR uses a Q-learning estimator $\hat{f}$ to reduce the variance in the first term. Its optimization often yields at least as good performance as either Q-learning or IPWE, as long as Q-learning has positive correlation with the true rewards.
However, vanilla DR can perform worse in extreme cases, as noticed by \cite{kang2007demystifying}.

We extend PIL-IML to DR, given that the probabilities are logged for every action candidate, $\mu(\vct a\mid\vct x_i),\forall\vct a\in\mathcal{A}(\vct x)$, including the non-chosen ones. The PIL-DR is
\begin{equation}
\max_\pi {\rm DR}_{\tau}(\pi) =
\frac{1}{n}
\sum_{i=1}^n 
\bar{w}_i
    \bigl(r_i - \hat{f}(\vct x_i,\vct a_i)\bigr)
    \label{eq:clipped-dr-2}
    +
    \frac{1}{n}
    \sum_{i=1}^n 
    \sum_{\vct a\in\mathcal{A}(\vct x_i)} \bar{\pi}(\vct a\mid\vct x_i)
    \hat{f}(\vct x_i,\vct a)
    + \bar R(\pi),
\end{equation}
where $\bar\pi=\bar{w}\mu$ is the lower-bounded policy induced by the lower-bounded weights and $\bar{R}(\pi)=\sum_{i=1}^n(w_i-\bar w_i)r_i$ reflects the bound error.
With nonnegative rewards, we optimize the lower-bound of DR by removing the $\bar{R}(\pi)\geq0$ term.

\section{Reproducing Criteo counterfactual analaysis}
\label{app:criteo}

\begin{table}[h]
\centering
\caption{Reproducing Criteo counterfactual analysis \citep{lefortier2016large}}
\label{tab:reproduce}
\begin{tabular}{
  l
  S[table-format=2.1(3)]
  S[table-format=3.1(2)]
  }
\toprule
{ Approach} & { IPWE ($\times 10^4$)} & { Gap (\%)} \\
\midrule
{ Logging policy} & 53.4 & 0.0  \\
{ Uniform random} & 44.7 \pm 2.1 & 1.7 \pm 2.1   \\
{ DM/Q-learning } & 49.6 \pm 1.7 & -0.5 \pm 2.0   \\
{ IPWE \citep{horvitz1952generalization} } & 49.9 \pm 1.8 & 0.1 \pm 1.6   \\
{ DR } & 53.7 \pm 14.4 & -1.5 \pm 5.7   \\
{ POEM \citep{swaminathan2015counterfactual} } & 52.7 \pm 1.6 & 0.3 \pm 0.6   \\
\bottomrule
\end{tabular}
\end{table}

A rerun of the provided scripts by \cite{lefortier2016large} yielded results in Table~\ref{tab:reproduce}.
None of the approaches outperformed the logging policy using the released models and features, likely due to the existence of confounding variables and model misspecifications.
Moreover, the confidence intervals are much smaller than their true values, as reflected in the inconsistency across data splits which is discussed in the main text and also show in Table~\ref{tab:criteo-reproduce2}\&\ref{tab:criteo-reproduce3}.
This is likely the combined results of fat-tailed importance weight distributions and a lack of weight-clipping in evaluations.

\begin{table}[h]
\centering
\caption{Reproduce Criteo counterfactual dataset large variance (DM/Q-Learning)}
\label{tab:criteo-reproduce2}
\begin{tabular}{
  c
  c
  S[table-format=2.1(3)]
  S[table-format=3.1(3)]
  c
  c
  c
  }
\toprule
{Validation Rank} & {Test Rank} & {Validation IPWE ($\times 10^4$)} & {Test IPWE ($\times 10^4$)} & {P} & {L} & {Lambda} \\
\midrule
1 & 2 & 52.5\pm3.1 & 49.6\pm1.7 & 1 & 0.1 & 0 \\
2 & 1 & 49.2\pm3.4 & 53.5\pm14.2 & 1 & 1 & 0  \\
3 & 14 & 47.4\pm7.0 & 41.0\pm4.4 & 1 & 0.1 & 1E-8	\\
\bottomrule
\end{tabular}
\end{table}

~

\begin{table}[h]
\centering
\caption{Reproduce Criteo counterfactual dataset large variance (IPS/IPWE)}
\label{tab:criteo-reproduce3}
\begin{tabular}{
  c
  c
  S[table-format=2.1(3)]
  S[table-format=3.1(3)]
  c
  c
  c
  }
\toprule
{Validation Rank} & {Test Rank} & {Validation IPWE ($\times 10^4$)} & {Test IPWE ($\times 10^4$)} & {P} & {L} & {Lambda} \\
\midrule
1 & 5.5 & 53.8\pm3.6 & 49.9\pm1.8 & 0 & 10 & 1E-8 \\
2 & 5.5 & 53.8\pm3.6 & 49.9\pm1.8 & 0 & 10 & 0  \\
3 & 4   & 53.8\pm3.6 & 50.0\pm1.8 & 0 & 10 & 1E-6	\\
10 & 1  & 50.8\pm2.1 & 53.7\pm3.0 & 0.5 & 1 & 0	\\
\bottomrule
\end{tabular}
\end{table}

\section{Subsampling bootstrap}
\label{app:bootstrap}

Subsampling bootstrap \citep{politis1994large,geyer20065601} is a statistical approach to estimate confidence intervals with minimal assumptions. It is especially suitable for heavy-tail distributions or dependent samples. Our goal is to estimate asymptotic quantiles on independent sample draws from heavy-tail distributions.

The key idea is to extend central limit theorem (CLT) convergence properties to a more general 
\begin{equation}
    n^{-\beta}(T_n(\theta)-\theta)\stackrel{\rm dist.}{\to} F,
    \label{eq:subsample}
\end{equation}
where $T_n(\theta)$ is the finite-sample estimator of parameter $\theta$, $\beta>0$ is the generalized convergence rate, and $F$ is any limiting distribution.
A trivial special case is CLT with $\beta=0.5$.
A nontrivial example where $\beta\neq0.5$ may be sample max estimator for the parameter of a uniform distribution, i.e., $T_n(\theta)=\max(X_1,\dots,X_n)$ where $X\sim{\rm Unif}(0, \theta)$.
In this case, simple algebra shows $n(\theta-T_n(\theta))\stackrel{dist.}{\to} {\rm Exp}(\theta)$, i.e., $\beta=1$.
There are also examples where $\beta<0.5$ for robust regression estimators.
Additionally, the mean of a Cauchy distribution does not exist, so $\beta\geq0$ for the sample mean estimator.
(The median estimator does converges to its mode parameter at $\beta=0.5$.)
The following approach has been verified with both uniform and Cauchy distributions.

We are interested in finding the correct rate of $\beta$ for the importance weight distribution, e.g., between the uniform policy and the logging policy, from Criteo counterfactual-analysis dataset.
Shown in Figure~\ref{fig:IPWE-weight} of the main text, the convergence rate of the raw weights (without clipping) must follow $\beta<0.5$.
We regressed $\beta$ using subsampling bootstrap with varying size $0\ll b\ll n$, where for each size, we bootstrapped $K=\num{10000}\sim\num{100000}$ subsamples and recorded the distribution from the resulting estimators, $S_b = \{T_b^{(k)}:k=1,\dots,K\}$.
Let $Q_{q}(S_b)$ be the distribution quantile at $q$,
\eqref{eq:subsample} suggests that \begin{equation}
    b^{-\beta}(Q_{q}(S_b)-T_n(\theta))
    \to
    b^{-\beta}(Q_{q}(S_b)-\theta)
    \to F^{-1}(q),
\end{equation}
where the difference $|T_n(\theta)-\theta|$ converges to zero at a faster rate $O(n^{-\beta}) < O(b^{-\beta})$ and can be ignored.
For this purpose, we took $n^{0.5}\leq b\leq n^{0.75}$.
Since the data is iid, we used subsampling with replacement to increase speed.

\begin{figure}[ht]
    \centering
    \begin{subfigure}[t]{0.48\textwidth}
        \centering
        \includegraphics[height=1.2in]{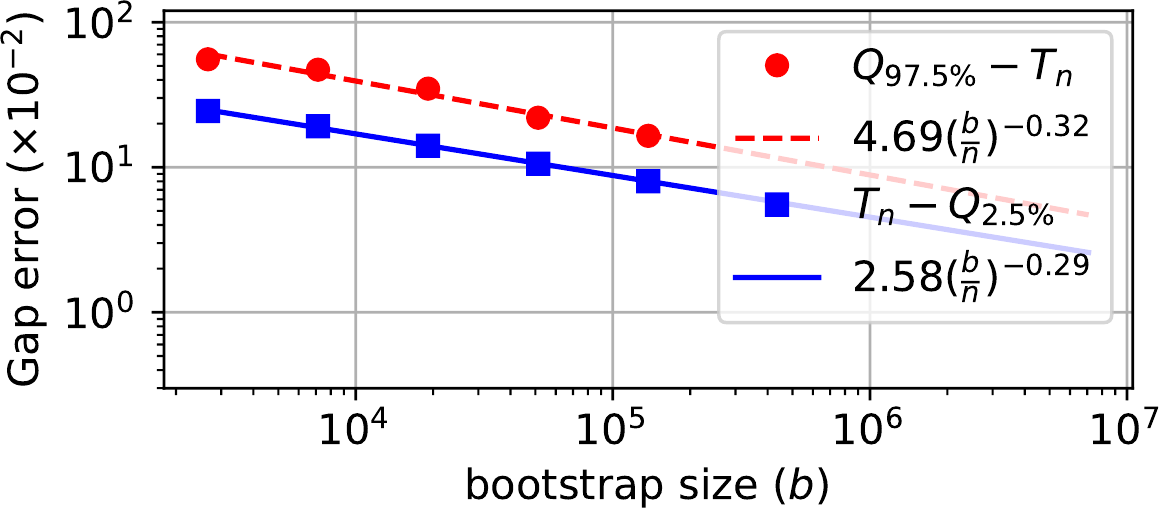}
        \caption{Slower rates without weight clipping.}
    \end{subfigure}%
    ~ 
    \begin{subfigure}[t]{0.48\textwidth}
        \centering
        \includegraphics[height=1.2in]{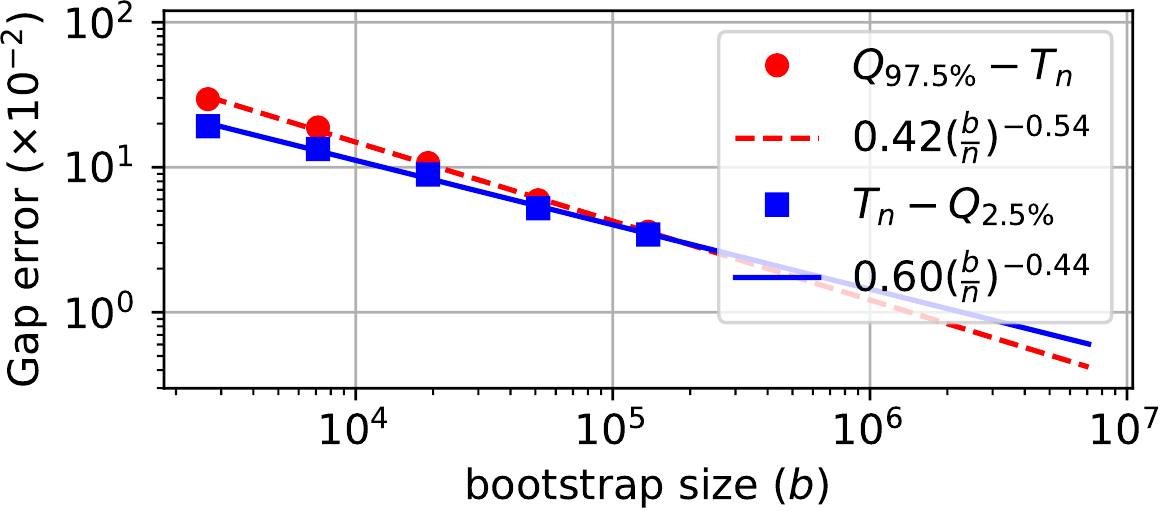}
        \caption{Faster rates with weight clipped (at 500).}
    \end{subfigure}
    \caption{Error bounds of the self-normalization Gap estimator, as a function of subsample size.
    Using subsampling bootstrap, the final error can be extrapolated with the correct rates.}
    \label{fig:sub_test-gap}
\end{figure}

Figure~\ref{fig:sub_test-gap} shows the log-log plots for the error bounds in the self-normalization estimator $T_b=\frac{1}{b}\sum_{i=1}^b \frac{\pi_{i_k}}{\mu_{i_k}}$, where $\{i_k:k=1,\dots,b\}$ is one subsample.
Without weight-clipping, the estimator converged at a slower rate $\beta\approx0.3$. With weight clipped at 500, the convergence rate was around $\beta=0.5$.
The final self-normalization quantity could be extrapolated as $(98.3-2.6, 98.3+4.7) = (95.7, 103.0)\%$ without clipping, i.e., 7.3\% estimation error.
On the other hand, weight clipping introduced 7$\sim$8\% bias,
which is on a similar scale, but significantly improves stability.
Intuitively, one click should neither be weighed more than 500 times.

\begin{figure}[ht]
    \centering
    \begin{subfigure}[t]{0.48\textwidth}
        \centering
        \includegraphics[height=1.2in]{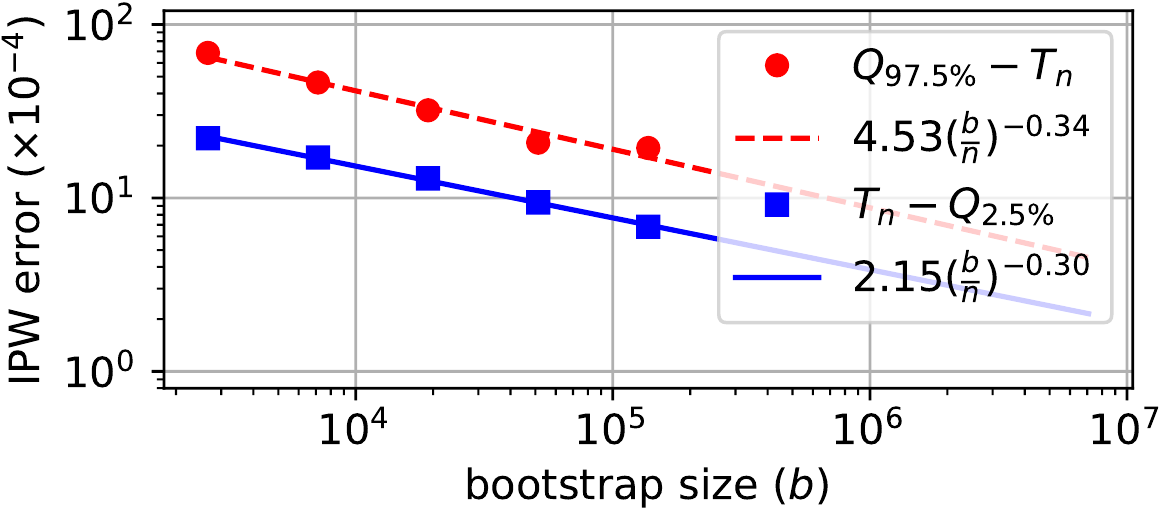}
        \caption{Slower rates without weight clipping.}
    \end{subfigure}%
    ~ 
    \begin{subfigure}[t]{0.48\textwidth}
        \centering
        \includegraphics[height=1.2in]{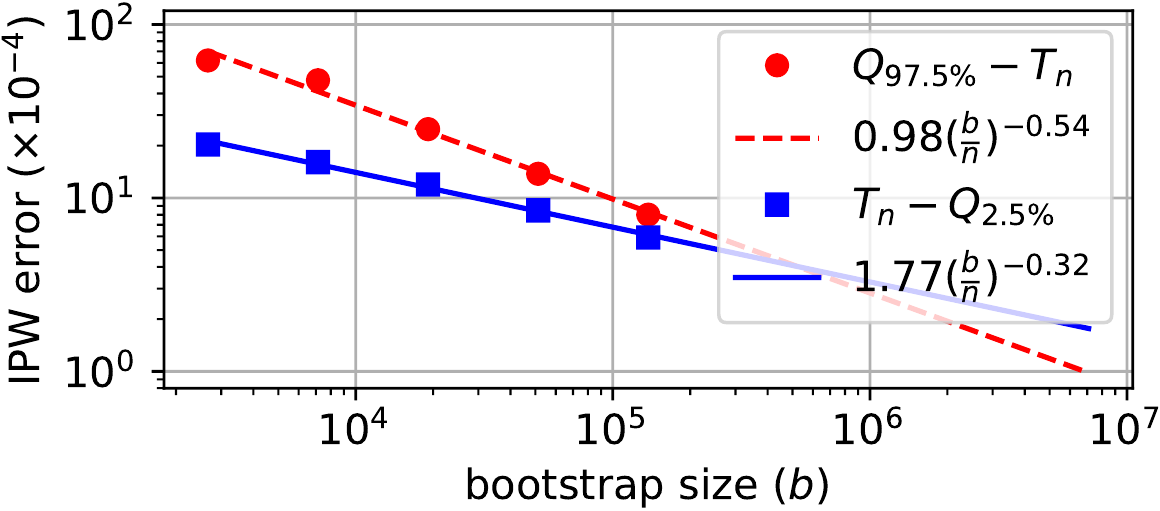}
        \caption{Faster rates with weight clipped (at 500).}
    \end{subfigure}
    \caption{Error bounds of IPWE as a function of subsample size. Using subsampling boostrap, the final error can be extrapolated with the correct rates.}
    \label{fig:sub_test-ipwe}
\end{figure}

Figure~\ref{fig:sub_test-ipwe} shows a similar story that weight clipping could improve convergence quality and yield tighter IPWE error bounds.
While most of the rate improvements were notable,
one particular interesting point is that weight-clipping at 500 did not seem to help the lower-bound estimation of IPWE, which still had a convergence rate at $\beta=0.32$.
One possible explanation could be that $\tau=500 \approx \sqrt{n\mathbb{E}R}$, i.e., square-root of the total number of clicks, which still left the worst-case IPWE variance potentially unbounded.

\begin{table}[ht]
    \caption{Offline estimates combining both IPWE and Gap errors. Weight clip contributed to smaller confidence intervals in the final offline estimates, without being overly pessimistic.}
    \label{tab:criteo-unif}
    \centering
    \begin{tabular}{c c c c c}
    \toprule
     Weight Clip  & Type & IPWE w/ Clip & Self-Normalization   & Final Offline Estimate \\
    $\tau$ &      &  $\times 10^4$ & 100\% & $\times 10^4$ \\
    \midrule
     \multirow{2}{*}{$\infty$} & Mean & 44.7         & 98.3          & 45.6  \\
                               & CI   & (42.5, 49.2) & (95.7, 103.0) & (38.2, 52.2) \\
    \midrule
     \multirow{2}{*}{500}      & Mean & 43.6         & 92.6          & 47.3 \\
                               & CI   & (41.8, 44.6) & (92.0, 93.0)  & (41.8, 52.6) \\
    \bottomrule
    \end{tabular}
\end{table}

The final offline estimator should include both IPWE and self-normalization Gap errors \citep{bottou2013counterfactual}. Since the global click rate is around 0.5\%, we may assume that the expected clicks in any self-normalization Gaps to be within $0\sim1\%$, i.e., at most twice their global average. Table~\ref{tab:criteo-unif} and Table~\ref{tab:criteo} in the main text both report the offline estimates using the following estimator:
\begin{equation}
    \mbox{Offline estimation} = {\rm IPWE} + {\rm Gap } R {\rm \; s.t. \; } 0<R<0.01.
\end{equation}

Judging by the amount of self-normalization Gaps, we found that clipped estimators to also be useful in (greedy application) of Q-learning (aka. direct method).
On the other hand, IML-based methods, such as IPWE$_\tau$, PIL, and POEM, tend to have lighter tails. We applied weight clipping to all methods for fairness, while the light-tail distributions tend to suffer smaller clipping biases.

\section{Second-Order Model and Implementation}

The second-order model which helped us reduce the IML Gap from 0.40 to 0.35 still follows an exponential family with potential function $\log\pi(a\mid x)\simeq\phi(x,a)$, but the potential scores become second-order:
\begin{equation}
    \phi(x,a) = x^\top UV^\top a + w^\top a,
\end{equation}
where $w\in\mathbb{R}^p$ is the first-order coefficient vector and $U,V\in\mathbb{R}^{p\times r}$ are second-order coefficient matrices.
We took $r=256$.
We also experimented with deeper models and nonlinear activation functions, but these approaches did not seem to improve IML much further.

While the IML loss improves 10\%, the offline click rate improved much less, around $0.4\times10^4$.
We used second-order models when available, except for the original first-order POEM for convenience. We believe its inferior performance may be partially due to lack of model depth as well as a possibility to yield policies that are more saturated, which may lead to inferior generalization.
It is common for a policy to become much more complex to yield minimal improvements, but part of our conclusion also suggests keeping simpler models to improve exploration properties.

\section{Results on other UCI datasets}
\label{app:uci}

Results on UCI datasets largely follow similar patterns in Figure~\ref{fig:uci}, with varying levels of benefits from variance reduction techniques.
The effectiveness of clipped-IPWE, clipped-DR, and IPWE-IML depends on the amount of misspecification that we can artificially inject using rank-2 models (which underfit the problems).
When the differences are small, we also observe that rank-2 models are sufficient to imitate the logging policy with near-zero IML losses.
It often implies the true models are considerably simpler, the biases are considerably smaller, and the optimal strategy might still be Q-learning.
The only negative evidence is from wdbc dataset, where the original problem was to identify benign against malign cancers but we modified the problem to classification of the actual cancer type, which are noisier labels.

\begin{figure*}[ht]
  \centering
  \begin{subfigure}[b]{0.48\linewidth}
    \centering\includegraphics[width=0.8\linewidth]{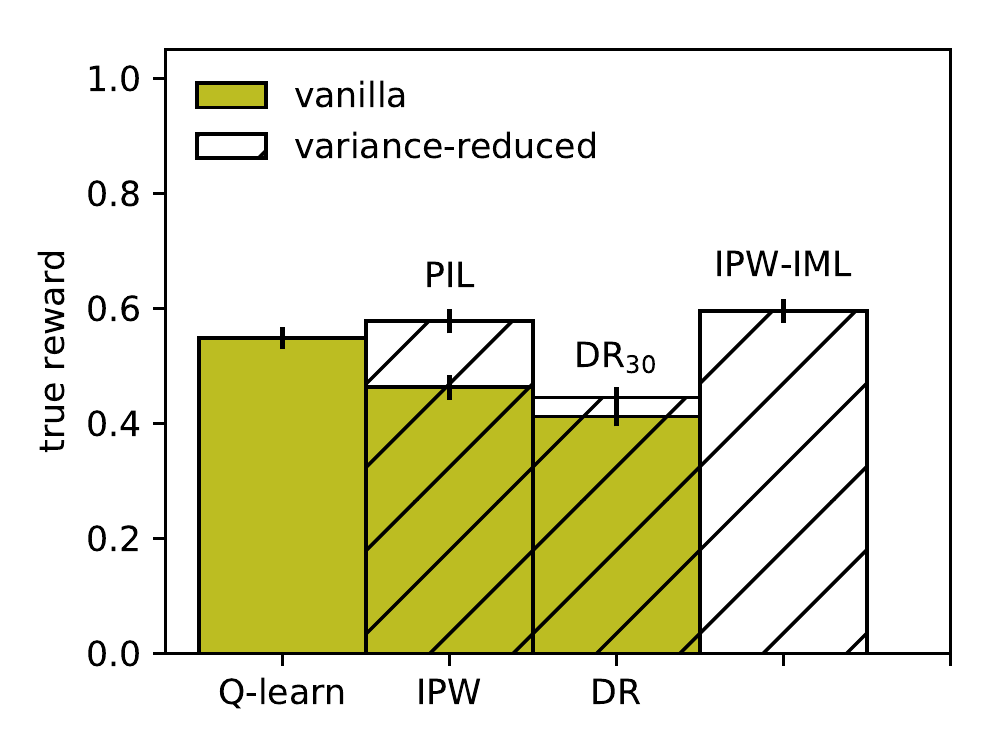}
    \caption{Variance reduction compared with vanilla methods.}
  \end{subfigure}\hspace{1.5em}%
  \begin{subfigure}[b]{0.48\linewidth}
    \centering\includegraphics[width=0.8\linewidth]{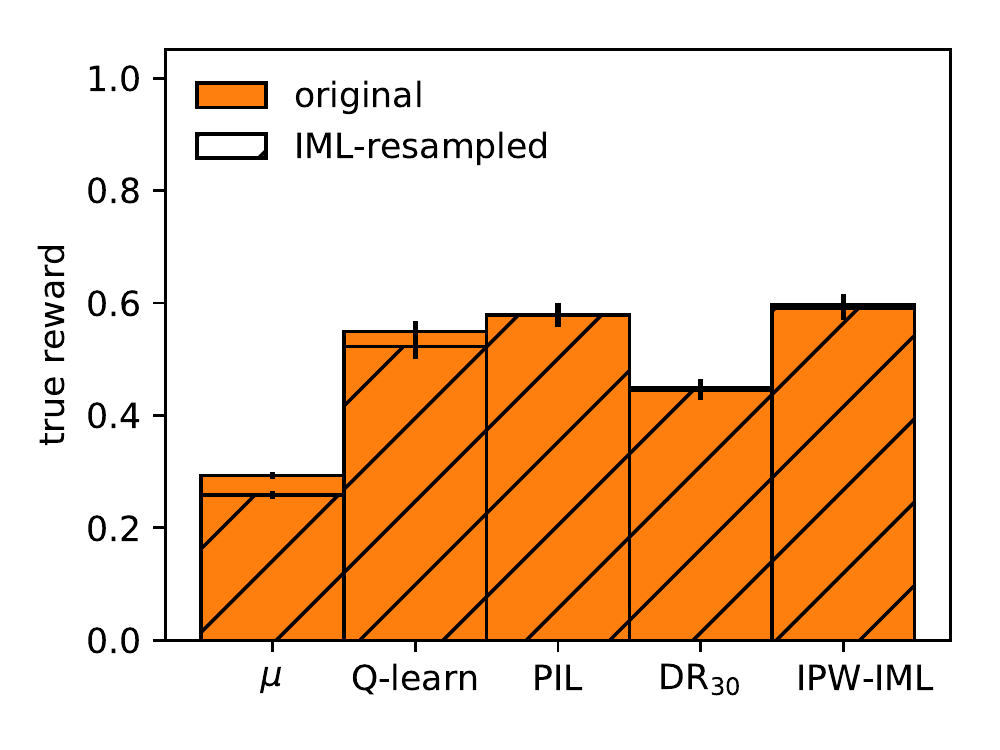}
    \caption{Online application of IML-resampling}
  \end{subfigure}\hspace{1.5em}%
  \caption{Multiclass-to-bandit conversion on UCI ecoli dataset.}
\end{figure*}

\begin{figure*}[ht]
  \centering
  \begin{subfigure}[b]{0.48\linewidth}
    \centering\includegraphics[width=0.8\linewidth]{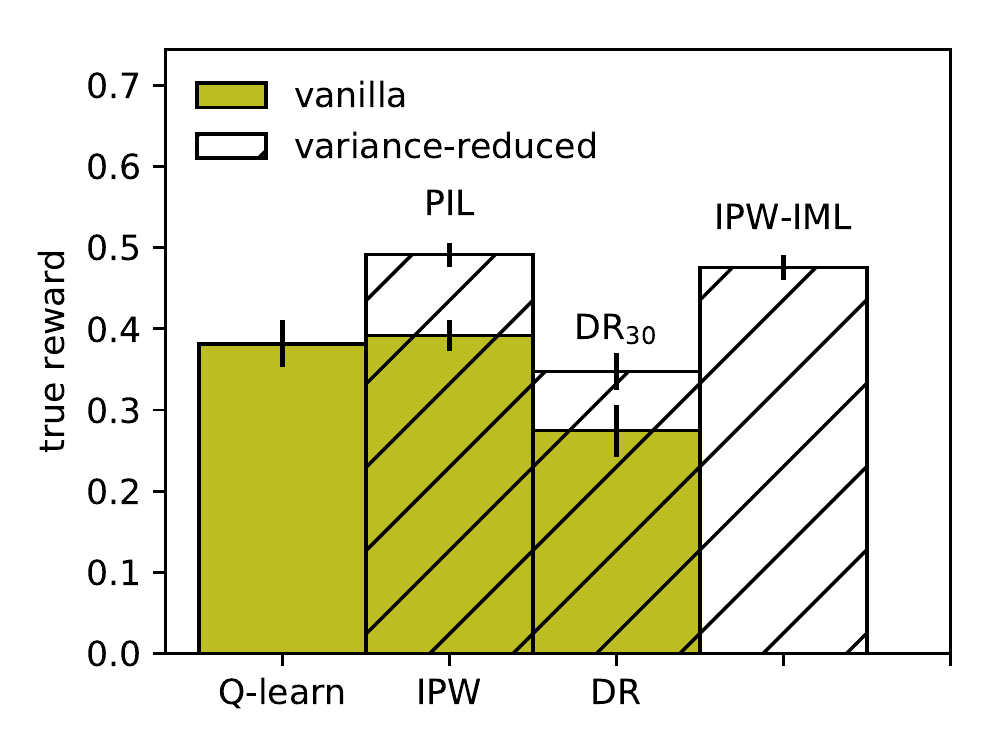}
    \caption{Variance reduction compared with vanilla methods.}
  \end{subfigure}\hspace{1.5em}%
  \begin{subfigure}[b]{0.48\linewidth}
    \centering\includegraphics[width=0.8\linewidth]{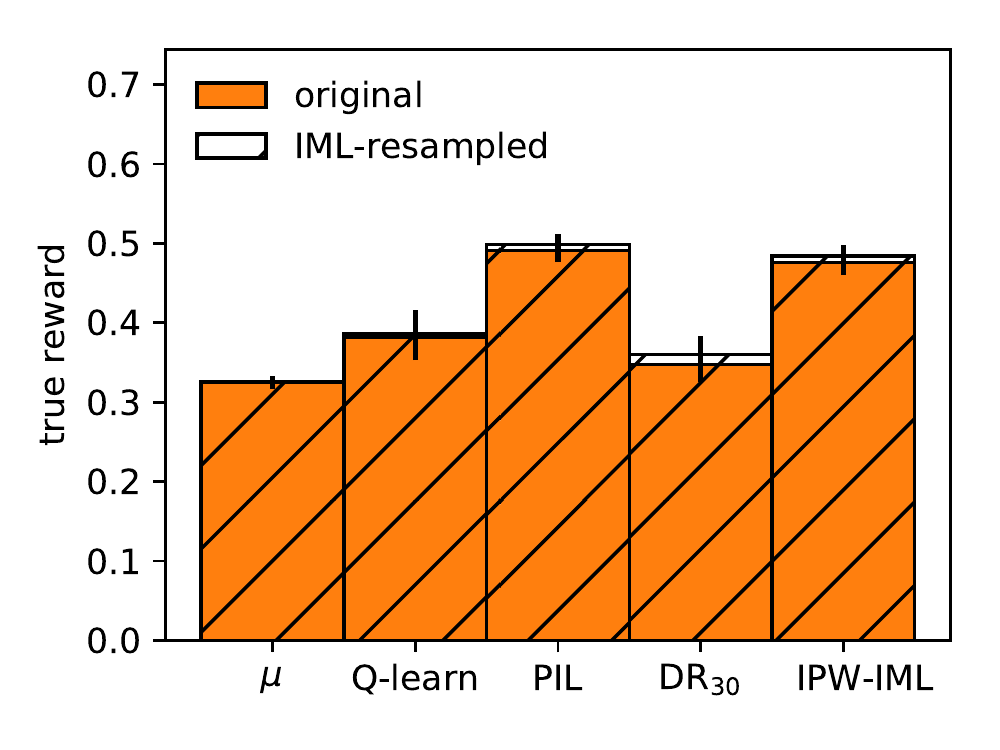}
    \caption{Online application of IML-resampling}
  \end{subfigure}\hspace{1.5em}%
  \caption{Multiclass-to-bandit conversion on UCI glass dataset.}
\end{figure*}

\begin{figure*}[ht]
  \centering
  \begin{subfigure}[b]{0.48\linewidth}
    \centering\includegraphics[width=0.8\linewidth]{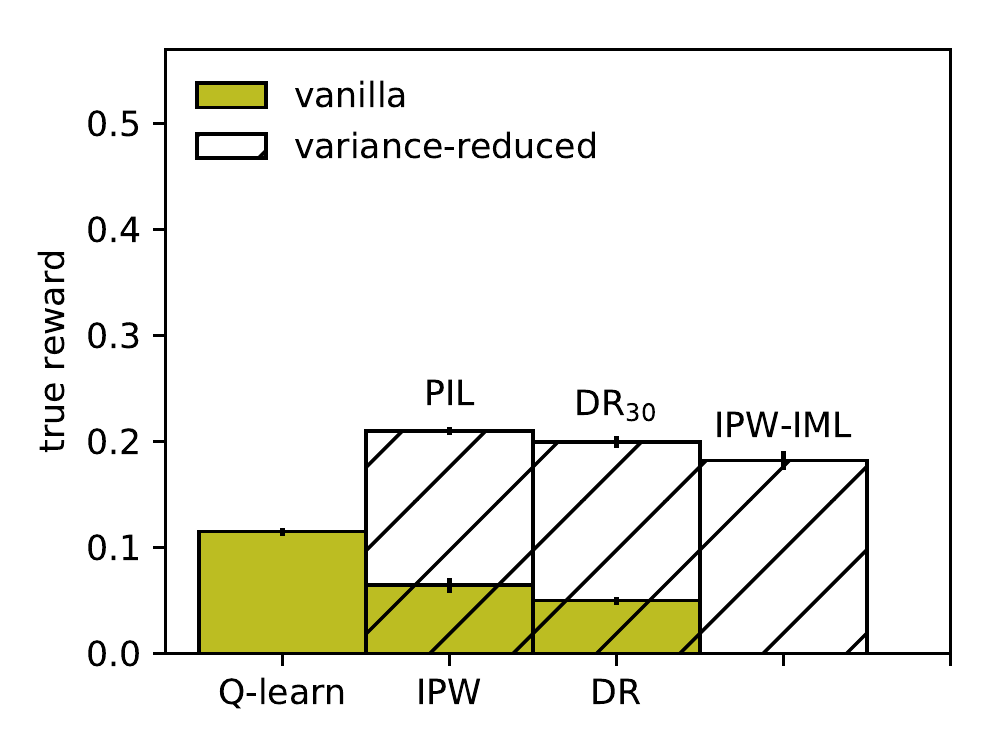}
    \caption{Variance reduction compared with vanilla methods.}
  \end{subfigure}\hspace{1.5em}%
  \begin{subfigure}[b]{0.48\linewidth}
    \centering\includegraphics[width=0.8\linewidth]{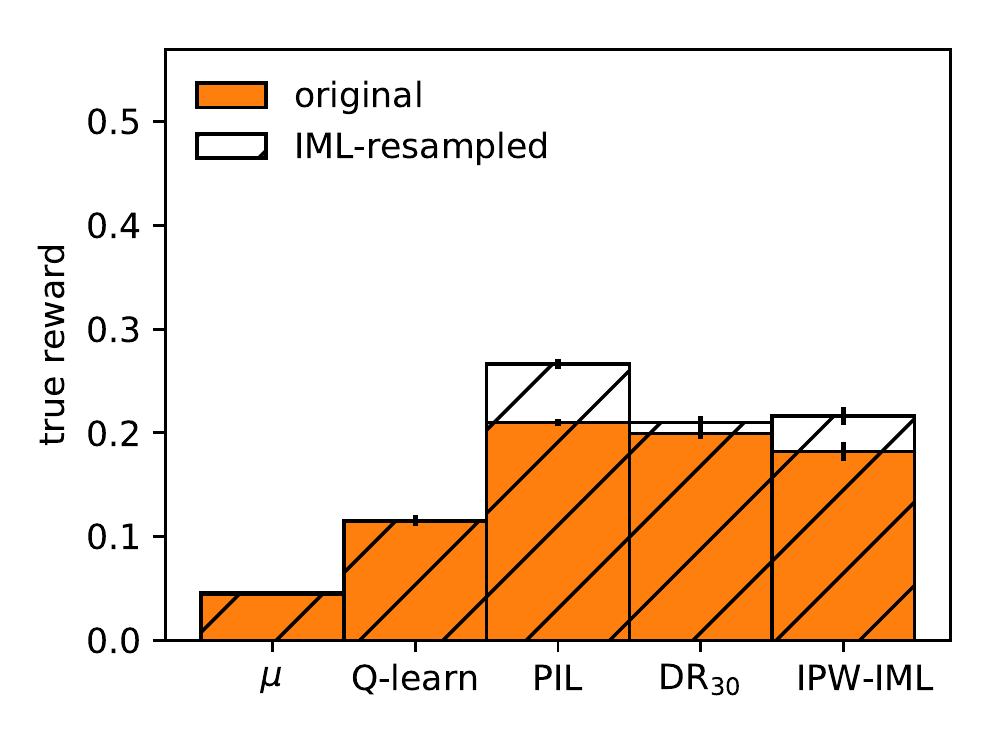}
    \caption{Online application of IML-resampling}
  \end{subfigure}\hspace{1.5em}%
  \caption{Multiclass-to-bandit conversion on UCI letter dataset.}
\end{figure*}

\begin{figure*}[ht]
  \centering
  \begin{subfigure}[b]{0.48\linewidth}
    \centering\includegraphics[width=0.8\linewidth]{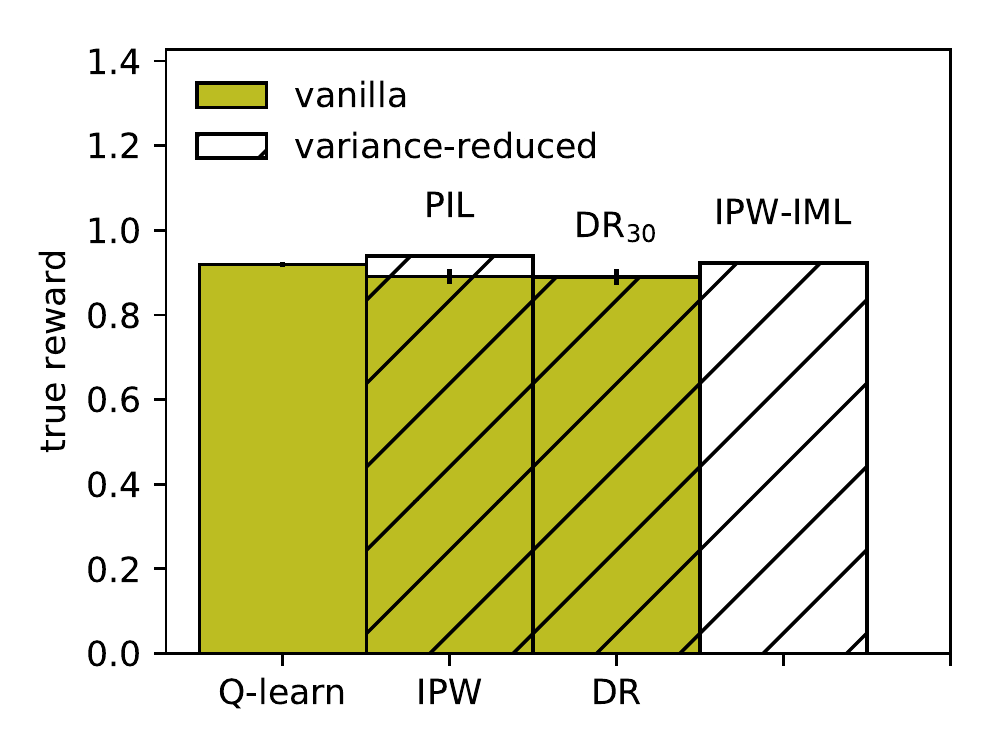}
    \caption{Variance reduction compared with vanilla methods.}
  \end{subfigure}\hspace{1.5em}%
  \begin{subfigure}[b]{0.48\linewidth}
    \centering\includegraphics[width=0.8\linewidth]{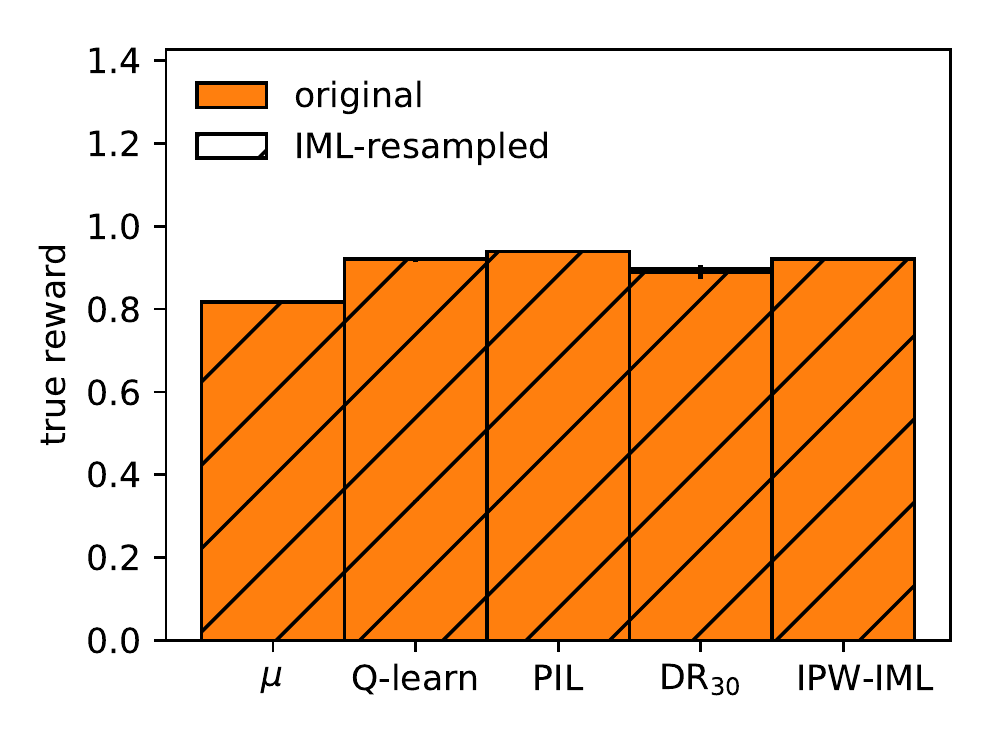}
    \caption{Online application of IML-resampling}
  \end{subfigure}\hspace{1.5em}%
  \caption{Multiclass-to-bandit conversion on UCI page-blocks dataset.}
\end{figure*}

\begin{figure*}[ht]
  \centering
  \begin{subfigure}[b]{0.48\linewidth}
    \centering\includegraphics[width=0.8\linewidth]{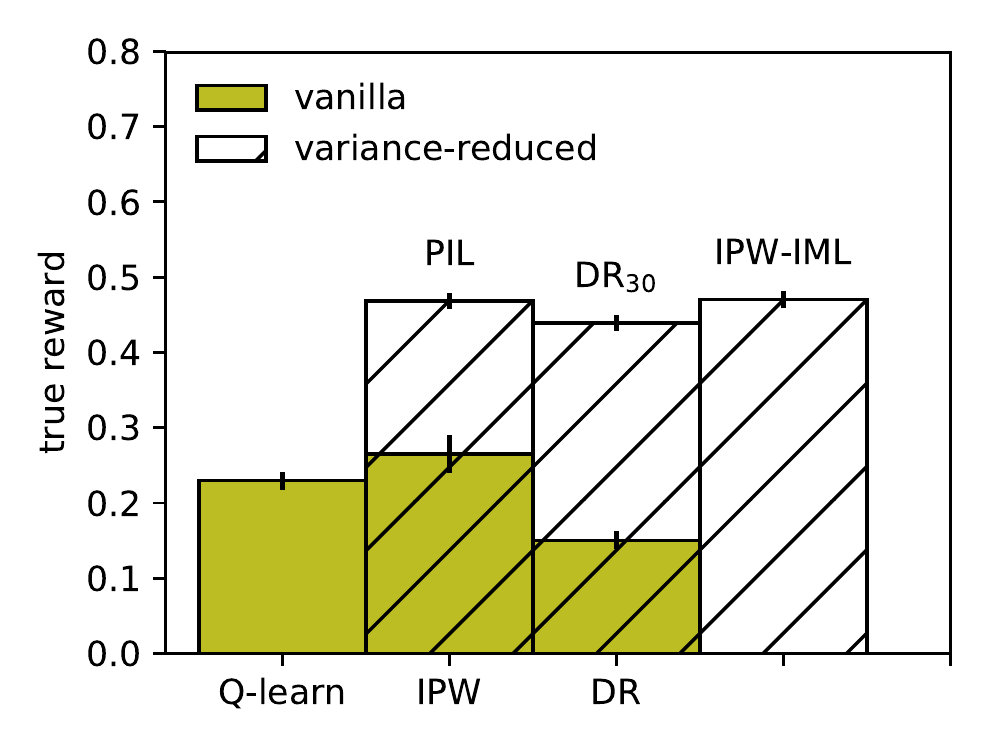}
    \caption{Variance reduction compared with vanilla methods.}
  \end{subfigure}\hspace{1.5em}%
  \begin{subfigure}[b]{0.48\linewidth}
    \centering\includegraphics[width=0.8\linewidth]{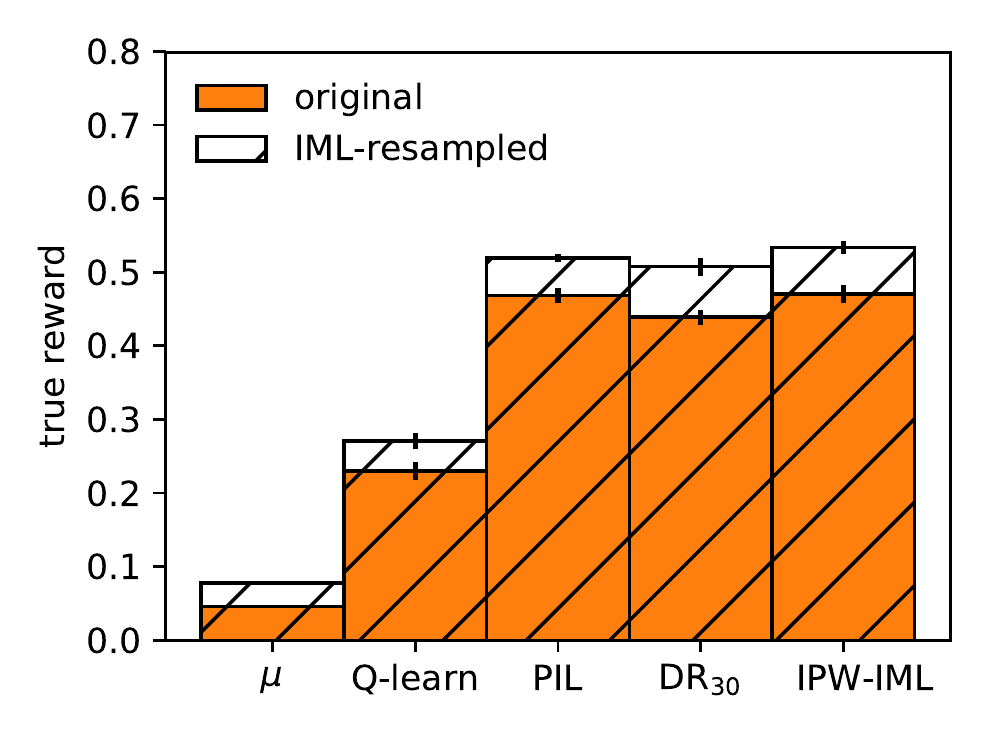}
    \caption{Online application of IML-resampling}
  \end{subfigure}\hspace{1.5em}%
  \caption{Multiclass-to-bandit conversion on UCI pendigits dataset.}
\end{figure*}

\begin{figure*}[ht]
  \centering
  \begin{subfigure}[b]{0.48\linewidth}
    \centering\includegraphics[width=0.8\linewidth]{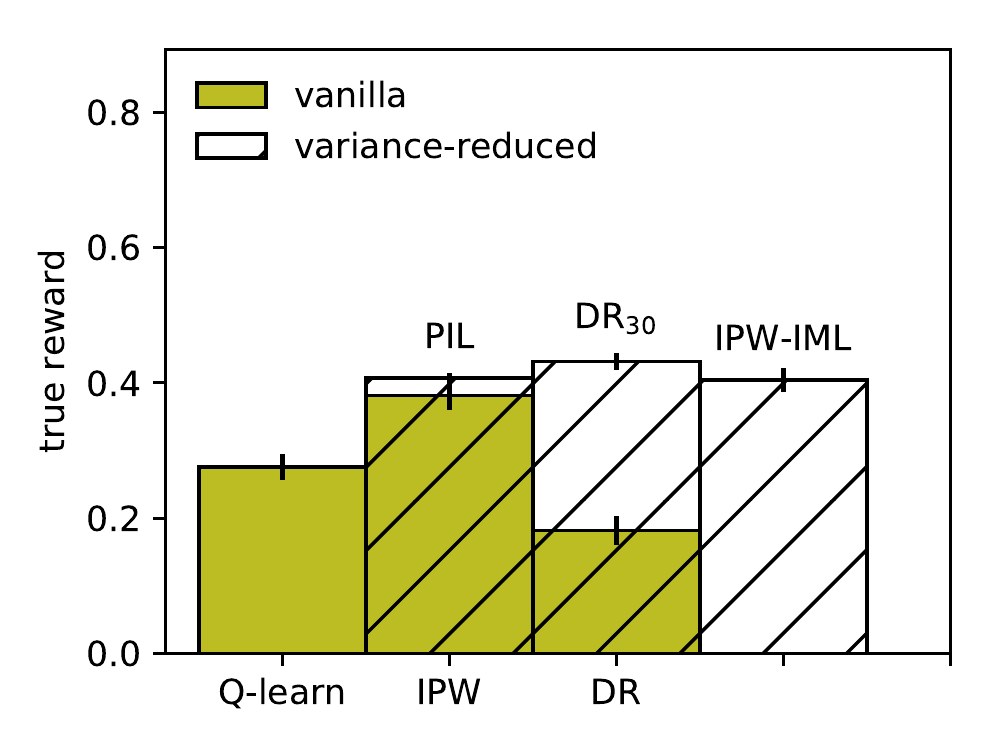}
    \caption{Variance reduction compared with vanilla methods.}
  \end{subfigure}\hspace{1.5em}%
  \begin{subfigure}[b]{0.48\linewidth}
    \centering\includegraphics[width=0.8\linewidth]{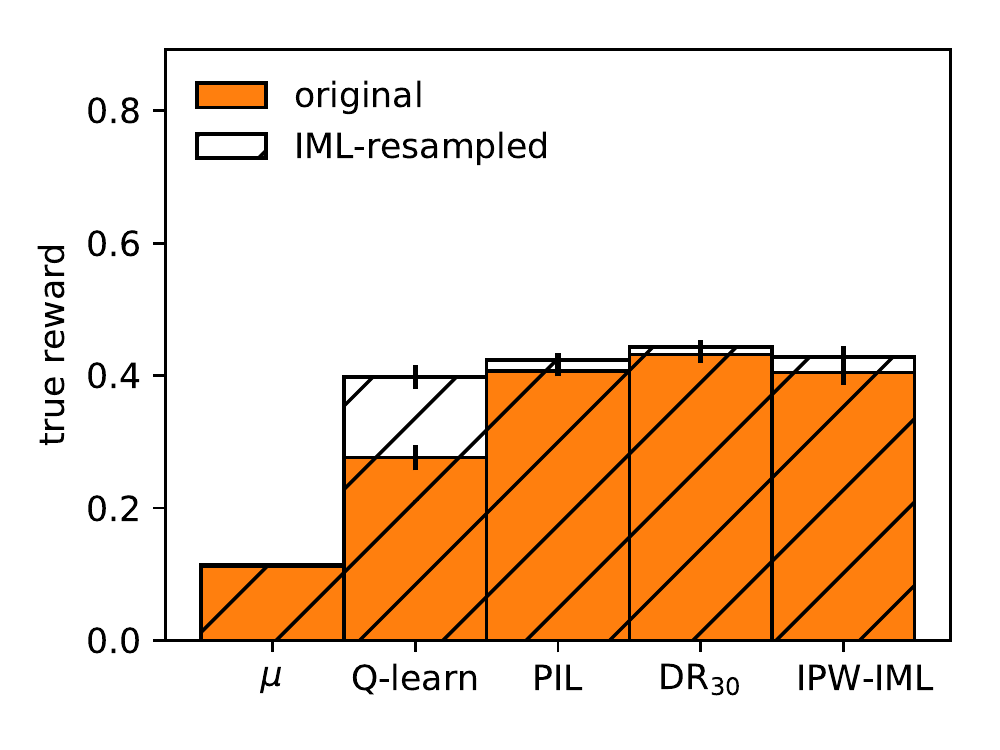}
    \caption{Online application of IML-resampling}
  \end{subfigure}\hspace{1.5em}%
  \caption{Multiclass-to-bandit conversion on UCI satimage dataset.}
\end{figure*}

\begin{figure*}[ht]
  \centering
  \begin{subfigure}[b]{0.48\linewidth}
    \centering\includegraphics[width=0.8\linewidth]{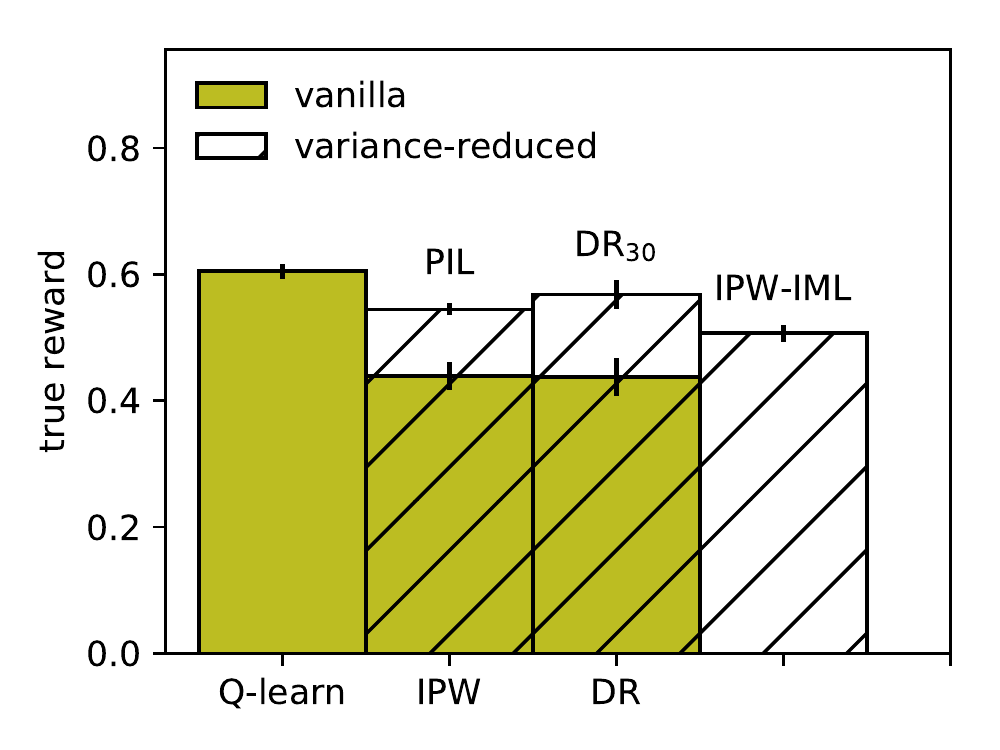}
    \caption{Variance reduction compared with vanilla methods.}
  \end{subfigure}\hspace{1.5em}%
  \begin{subfigure}[b]{0.48\linewidth}
    \centering\includegraphics[width=0.8\linewidth]{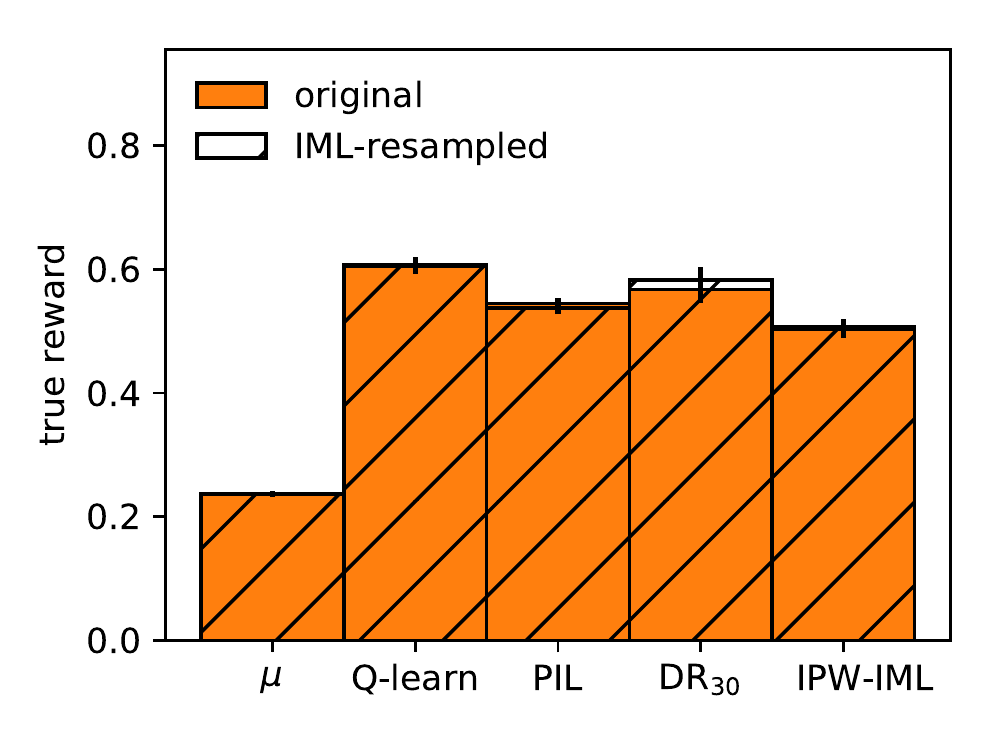}
    \caption{Online application of IML-resampling}
  \end{subfigure}\hspace{1.5em}%
  \caption{Multiclass-to-bandit conversion on UCI vehicle dataset.}
\end{figure*}

\begin{figure*}[ht]
  \centering
  \begin{subfigure}[b]{0.48\linewidth}
    \centering\includegraphics[width=0.8\linewidth]{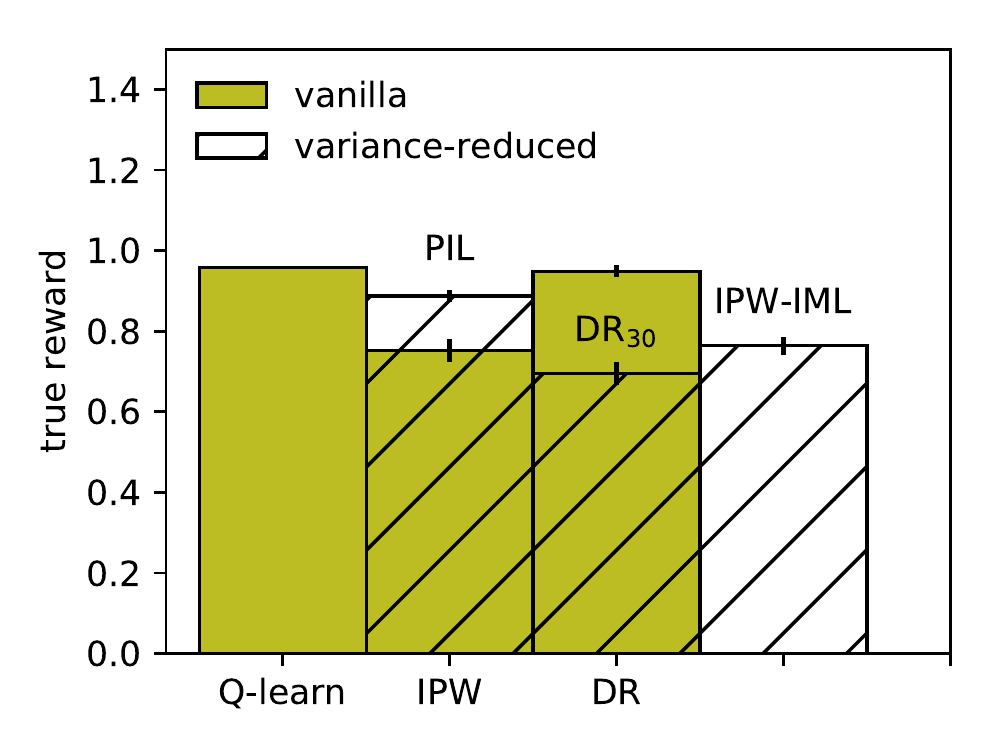}
    \caption{Variance reduction compared with vanilla methods.}
  \end{subfigure}\hspace{1.5em}%
  \begin{subfigure}[b]{0.48\linewidth}
    \centering\includegraphics[width=0.8\linewidth]{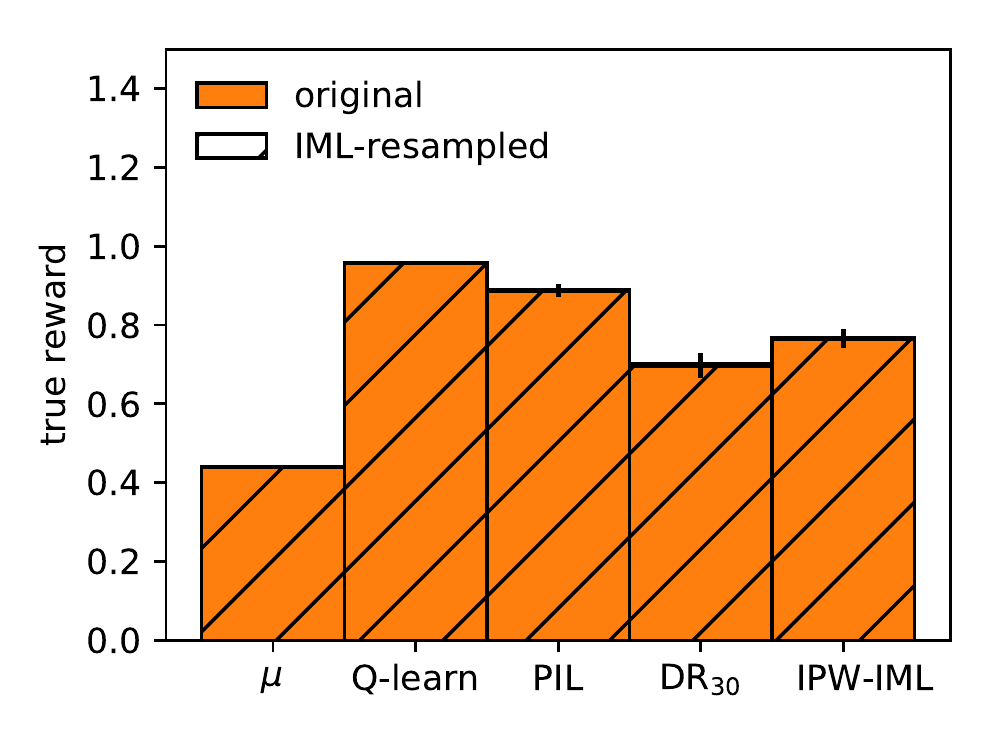}
    \caption{Online application of IML-resampling}
  \end{subfigure}\hspace{1.5em}%
  \caption{Multiclass-to-bandit conversion on UCI wdbc dataset.}
\end{figure*}

\begin{figure*}[ht]
  \centering
  \begin{subfigure}[b]{0.48\linewidth}
    \centering\includegraphics[width=0.8\linewidth]{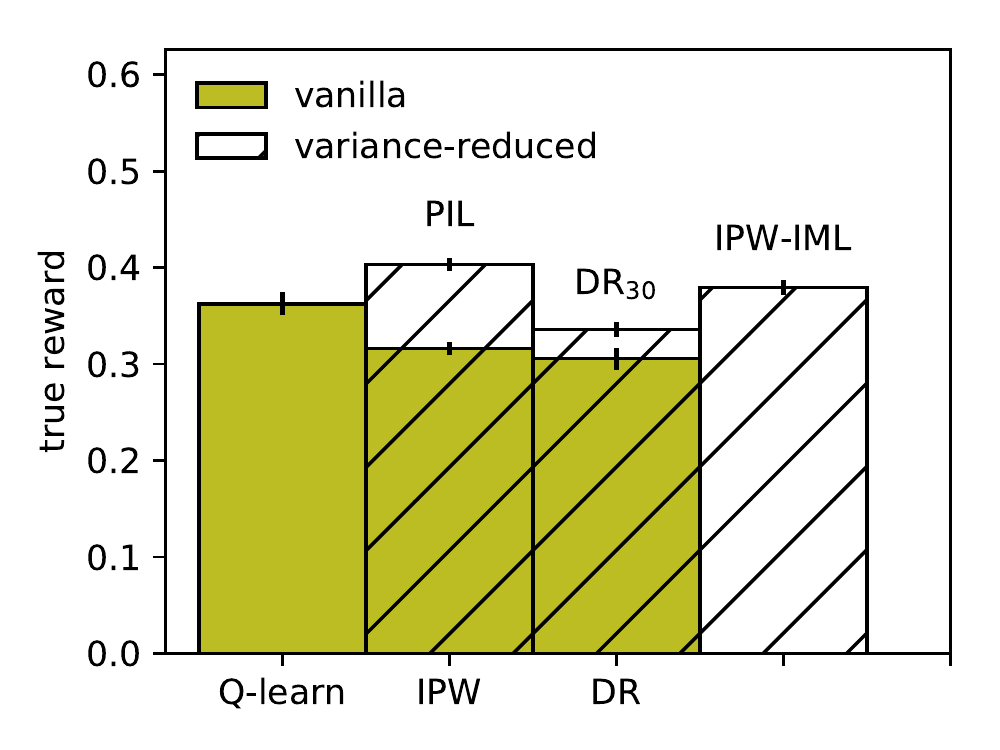}
    \caption{Variance reduction compared with vanilla methods.}
  \end{subfigure}\hspace{1.5em}%
  \begin{subfigure}[b]{0.48\linewidth}
    \centering\includegraphics[width=0.8\linewidth]{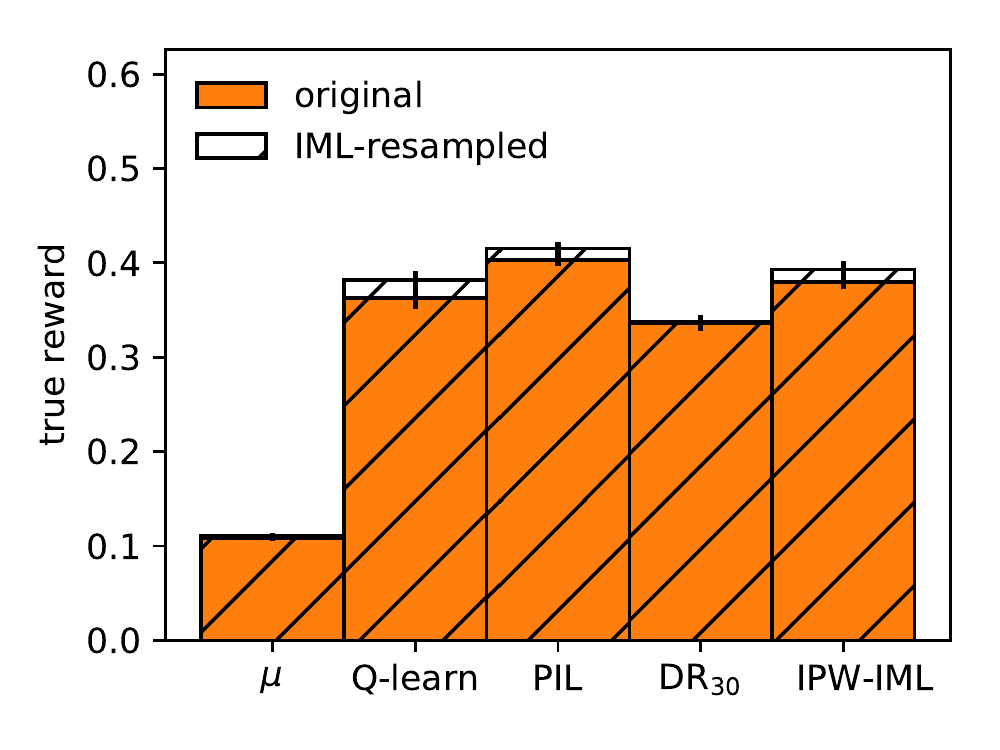}
    \caption{Online application of IML-resampling}
  \end{subfigure}\hspace{1.5em}%
  \caption{Multiclass-to-bandit conversion on UCI yeast dataset.}
\end{figure*}